\documentclass[11pt]{article} 
\usepackage{hyperref}
\usepackage{url}
\usepackage{fullpage}

\usepackage{amsmath,amsthm,amssymb}
\usepackage{algorithm,algorithmic}
\usepackage{stmaryrd,dsfont}
\usepackage[pdftex]{graphicx}
\usepackage{subfigure}
\usepackage{hyperref}
\usepackage[normalem]{ulem}
\usepackage{color}
\usepackage{comment}
\usepackage{xspace}
\usepackage{enumerate}
\usepackage{paralist}

\newtheorem{lem}{Lemma}
\newtheorem{thm}[lem]{Theorem}

\newtheorem{defn}[lem]{Definition}

\makeatletter
\newcommand{\newreptheorem}[2]{\newtheorem*{rep@#1}{\rep@title} 
\newenvironment{rep#1}[1]{\def\rep@title{#2 \ref*{##1}}\begin{rep@#1}}{\end{rep@#1}}
}
\makeatother

\newcommand*\rfrac[2]{{}^{#1}\!/_{#2}}

\newreptheorem{lem}{Lemma}
\newreptheorem{thm}{Theorem}
\newreptheorem{clm}{Claim}

\newcommand{\R}{{\mathbb R}}

\newcommand{\RR}{{\cal R}}

\renewcommand{\L}{{\cal L}}
\newcommand{\T}{{\mathbb T}}
\newcommand{\hT}{\hat{\mathbb T}}

\newcommand{\W}{{\cal W}}
\newcommand{\Q}{{\cal Q}}

\renewcommand{\P}{{\cal P}}

\newcommand{\ZZ}{{\cal Z}}

\newcommand{\bc}[1]{\left\{{#1}\right\}}
\newcommand{\br}[1]{\left({#1}\right)}
\newcommand{\bs}[1]{\left[{#1}\right]}
\newcommand{\abs}[1]{\left| {#1} \right|}
\newcommand{\norm}[1]{\left\| {#1} \right\|}

\newcommand{\bsd}[1]{\left\llbracket{#1}\right\rrbracket}

\renewcommand{\O}[1]{{\cal O}\br{{#1}}}
\newcommand{\softO}[1]{\tilde{\cal O}\br{{#1}}}
\newcommand{\Om}[1]{\Omega\br{{#1}}}

\newcommand{\ip}[2]{\left\langle{#1},{#2}\right\rangle}

\renewcommand{\vec}[1]{{\mathbf{#1}}}

\newcommand{\vecx}{\vec{x}}
\newcommand{\vecy}{\vec{y}}

\newcommand{\vecw}{\vec{w}}

\newcommand{\Zb}{\vec{Z}}

\newcommand{\veczero}{\vec{0}}

\newcommand{\x}{\vec{x}}
\newcommand{\bx}{\bar\x}
\newcommand{\by}{\bar y}
\newcommand{\w}{\vec{w}}
\renewcommand{\v}{\vec{v}}

\newcommand{\Ind}[1]{{\mathbb I}\bs{{#1}}}

\renewcommand{\th}{^{\text{th}}}

\newcommand{\spmb}{{\bfseries 1PMB}\xspace}
\newcommand{\tpmb}{{\bfseries 2PMB}\xspace}
\newcommand{\preck}{Prec$_{@k}$\xspace}

\graphicspath{{.}{figs/}}

\title{Online and Stochastic Gradient Methods for Non-decomposable Loss Functions}

\author{
Purushottam Kar$^\ast$ \quad Harikrishna Narasimhan$^\dagger$ \quad Prateek Jain$^\ast$\\
$^\ast$Microsoft Research, INDIA\\
$^\dagger$Indian Institute of Science, Bangalore, INDIA\\
\texttt{\{t-purkar,prajain\}@microsoft.com, harikrishna@csa.iisc.ernet.in}
}

\begin{document}

\maketitle

\begin{abstract}
Modern applications in sensitive domains such as biometrics and medicine frequently require the use of \emph{non-decomposable} loss functions such as precision$@k$, F-measure etc. Compared to point loss functions such as hinge-loss, these offer much more fine grained control over prediction, but at the same time present novel challenges in terms of algorithm design and analysis. In this work we initiate a study of online learning techniques for such non-decomposable loss functions with an aim to enable incremental learning as well as design scalable solvers for batch problems. To this end, we propose an online learning framework for such loss functions. Our model enjoys several nice properties, chief amongst them being the existence of efficient online learning algorithms with sublinear regret and online to batch conversion bounds. Our model is a provable extension of existing online learning models for point loss functions. We instantiate two popular losses, \preck and pAUC, in our model and prove sublinear regret bounds for both of them. Our proofs require a novel structural lemma over ranked lists which may be of independent interest. We then develop scalable stochastic gradient descent solvers for non-decomposable loss functions. We show that for a large family of loss functions satisfying a certain uniform convergence property (that includes \preck, pAUC, and F-measure), our methods provably converge to the empirical risk minimizer. Such uniform convergence results were not known for these losses and we establish these using novel proof techniques. We then use extensive experimentation on real life and benchmark datasets to establish that our method can be orders of magnitude faster than a recently proposed cutting plane method.
\end{abstract}

\section{Introduction}
Modern learning applications frequently require a level of fine-grained control over prediction performance that is not offered by traditional ``per-point'' performance measures such as hinge loss. Examples include datasets with mild to severe label imbalance such as spam classification wherein positive instances (spam emails) constitute a tiny fraction of the available data, and learning tasks such as those in medical diagnosis which make it imperative for learning algorithms to be sensitive to class imbalances. Other popular examples include ranking tasks where precision in the top ranked results is valued more than overall precision/recall characteristics.
The performance measures of choice in these situations are those that evaluate algorithms over the entire dataset in a holistic manner. Consequently, these measures are frequently \emph{non-decomposable} over data points. More specifically, for these measures, the loss on a set of points cannot be expressed as the sum of losses on individual data points (unlike hinge loss, for example). Popular examples of such measures include F-measure, Precision$@k$, (partial) area under the ROC curve etc.

Despite their success in these domains, non-decomposable loss functions are not nearly as well understood as their decomposable counterparts. The study of point loss functions has led to a deep understanding about their behavior in batch and online settings and tight characterizations of their generalization abilities. The same cannot be said for most non-decomposable losses. For instance, in the popular online learning model, it is difficult to even instantiate a non-decomposable loss function as defining the per-step penalty itself becomes a challenge.

\vspace*{-0.5ex}

\subsection{Our Contributions}
Our first main contribution is a framework for online learning with non-decomposable loss functions. The main hurdle in this task is a proper definition of instantaneous penalties for non-decomposable losses. Instead of resorting to canonical definitions, we set up our framework in a principled way that fulfills the objectives of an online model. Our framework has a very desirable characteristic that allows it to recover existing online learning models when instantiated with point loss functions. Our framework also admits online-to-batch conversion bounds.

We then propose an efficient Follow-the-Regularized-Leader \cite{Rakhlin09} algorithm within our framework. We show that for loss functions that satisfy a generic ``stability'' condition, our algorithm is able to offer vanishing $\O{\frac{1}{\sqrt T}}$ regret. Next, we instantiate within our framework, convex surrogates for two popular performances measures namely, Precision at $k$ (\preck) and partial area under the ROC curve (pAUC) \cite{NarasimhanA13a} and show, via a stability analysis, that we do indeed achieve sublinear regret bounds for these loss functions. Our stability proofs involve a structural lemma on sorted lists of inner products which proves Lipschitz continuity properties for measures on such lists (see Lemma~\ref{lem:list}) and might be useful for analyzing non-decomposable loss functions in general.

A key property of online learning methods is their applicability in designing solvers for offline/batch problems. With this goal in mind, we design a stochastic gradient-based solver for non-decomposable loss functions. Our methods apply to a wide family of loss functions (including \preck, pAUC and F-measure) that were introduced in \cite{Joachims05} and have been widely adopted \cite{YueFRJ07,ChakrabartiKSB08,McFeeL10} in the literature.
We design several variants of our method and show that our methods provably converge to the empirical risk minimizer of the learning instance at hand. Our proofs involve uniform convergence-style results which were not known for the loss functions we study and require novel techniques, in particular the structural lemma mentioned above.

Finally, we conduct extensive experiments on real life and benchmark datasets with pAUC and \preck as performance measures. We compare our methods to state-of-the-art methods that are based on cutting plane techniques \cite{NarasimhanA13b}. The results establish that our methods can be significantly faster, all the while offering comparable or higher accuracy values. For example, on a KDD 2008 challenge dataset, our method was able to achieve a pAUC value of 64.8\% within 30ms whereas it took the cutting plane method more than 1.2 seconds to achieve a comparable performance.




\vspace*{-0.5ex}

\subsection{Related Work}
Non decomposable loss functions such as \preck, (partial) AUC, F-measure etc, owing to their demonstrated ability to give better performance in situations with label imbalance etc, have generated significant interest within the learning community. From their role in early works as indicators of performance on imbalanced datasets \cite{KubatM97}, their importance has risen to a point where they have become the learning objectives themselves. Due to their complexity, methods that try to indirectly optimize these measures are very common e.g. \cite{DembczynskiWCH11}, \cite{YeCLC12} and \cite{DembczynskiJKWH13} who study the F-measure. However, such methods frequently seek to learn a complex probabilistic model, a task arguably harder than the one at hand itself. On the other hand are algorithms that perform optimization directly via structured losses. Starting from the seminal work of \cite{Joachims05}, this method has received a lot of interest for measures such as the F-measure \cite{Joachims05}, average precision \cite{YueFRJ07}, pAUC \cite{NarasimhanA13b} and various ranking losses \cite{ChakrabartiKSB08,McFeeL10}. These formulations typically use cutting plane methods to design dual solvers.

We note that the learning and game theory communities are also interested in non-additive notions of regret and utility. In particular \cite{RakhlinST11} provides a generic framework for online learning with non-additive notions of regret with a focus on showing regret bounds for mixed strategies in a variety of problems. However, even polynomial time implementation of their strategies is difficult in general. Our focus, on the other hand, is on developing efficient online algorithms that can be used to solve large scale batch problems. Moreover, it is not clear how (if at all) can the loss functions considered here (such as \preck) be instantiated in their framework.

Recently, online learning for AUC maximization has received some attention \cite{KarSJK13,oam-icml}. Although AUC is not a point loss function, it still decomposes over pairs of points in a dataset, a fact that \cite{KarSJK13} and \cite{oam-icml} crucially use. The loss functions in this paper do not exhibit any such decomposability.


\section{Problem Formulation}
\label{sec:formulation}
Let $\x_{1:t}:=\{\x_1, \dots, \x_t\}$, $\x_i \in \R^d$ and $y_{1:t}:=\{y_1, \dots, y_t\}$, $y_i \in \bc{-1,1}$ be the observed data points and {\em true} binary labels. We will use $\widehat{y}_{1:t}:=\{\widehat{y}_1, \dots, \widehat{y}_t\}$, $\widehat{y}_i \in \R$ to denote the predictions of a learning algorithm. We shall, for sake of simplicity, restrict ourselves to linear predictors $\widehat{y}_i = \w^\top\x_i$ for parameter vectors $\w\in\R^d$. A performance measure $\P: \{-1, 1\}^t \times  \R^t \rightarrow \R_+$ shall be used to evaluate the the predictions of the learning algorithm against the true labels. Our focus shall be on non-decomposable performance measures such as \preck, partial AUC etc.

Since these measures are typically non-convex, convex surrogate {\em loss functions} are used instead (we will use the terms \emph{loss function} and \emph{performance measure} interchangeably). A popular technique for constructing such loss functions is the \emph{structural SVM} formulation \cite{Joachims05} given below. For simplicity, we shall drop mention of the training points and use the notation $\ell_{\P}(\w) := \ell_{\P}(\vecx_{1:T}, y_{1:T}, \vecw)$.
\begin{equation}
  \label{eq:ssvm_loss}
  \ell_{\P}(\vecw)=\max_{\bar{\vecy}\in \{-1, +1\}^T} \sum_{i=1}^T (\bar{y}_i - y_i) \vecx_i^\top\vecw - \P(\bar{\vecy}, \vecy). 
\end{equation}
\textbf{Precision$\mathbf{@k}$.} The \preck measure ranks the data points in order of the predicted scores $\widehat{y}_i$ and then returns the number of true positives in the top ranked positions. This is valuable in situations where there are very few positives. To formalize this, for any predictor $\w$ and set of points $\x_{1:t}$, define $S(\x,\w):=\{j : \w^\top\x>\w^\top\x_j\}$ to be the set of points which $\w$ ranks above $\x$. Then define
\begin{equation}
  \label{eq:ind_pauc}
  {\mathbb T}_{\beta, t}(\vecx, \vecw)=\begin{cases} 1,&\ \text{ if }\ |S(\x,\w)|< \lceil\beta t\rceil,\\
0, &\ \text{otherwise}.\end{cases}
\end{equation}
i.e. ${\mathbb T}_{\beta, t}(\vecx, \vecw)$ is non-zero iff $\vecx$ is in the top-$\beta$ fraction of the set. Then we define\footnote{An equivalent definition considers $k$ to be the \emph{number} of top ranked points instead.}
\[
\text{\preck}(\w) := \sum_{j:\T_{k,t}(\x_j,\w) = 1}\Ind{y_j = 1}.
\]
The structural surrogate for this measure is then calculated as \footnote{\cite{Joachims05} uses a slightly modified, but equivalent, definition that considers labels to be Boolean.}
\begin{equation}
  \label{eq:prec}
  \ell_{\text{\preck}}(\vecw) = \max_{\substack{\bar{\vecy}\in \{-1, +1\}^t\\\sum_i (\bar{y}_i+1)=2kt}} \sum_{i=1}^t (\bar{y}_i-y_i) \vecx_i^T\vecw - \sum_{i=1}^t y_i\bar{y}_i. 
\end{equation}
\textbf{Partial AUC.} This measures the area under the ROC curve with the false positive rate restricted to the range $[0,\beta]$. This is in contrast to AUC that allows false positive range in $[0,1]$. pAUC is useful in medical applications such as cancer detection where a small false positive rate is desirable. Let us extend notation to use $\T_{\beta,t}^-(\x,\w)$ to denote the indicator that selects the top $\beta$ fraction of the \emph{negatively} labeled points i.e. $\T_{\beta,t}^-(\x,\w) = 1$ iff $\abs{\bc{j : y_j < 0, \w^\top\x>\w^\top\x_j}} \leq \lceil\beta t_-\rceil$ where $t_-$ is the number of negatives. Then we define
\begin{equation}
\text{pAUC}(\w) = \sum_{i:y_i > 0}\sum_{j:y_j<0}\mathbb{T}^-_{\beta, t}(\x_j, \w)\\ \cdot {\mathbb I}[\x_i^\top\w \geq \x_j^\top\w].
\end{equation}
The structural surrogate for this performance measure can be equivalently expressed in a simpler form by replacing the indicator functions $\Ind{\cdot}$ with hinge loss as follows (see \cite{NarasimhanA13b}, Theorem 4)
\begin{equation}
\label{eq:pauc_cvx}
\ell_{\text{pAUC}}(\w) = \sum_{i:y_i>0}\sum_{j:y_j<0}\mathbb{T}^-_{\beta, t}(\vecx_j, \vecw)\cdot h(\vecx_i^\top\vecw- \vecx_j^\top\vecw),
\end{equation}
where $h(c)=\max(0, 1-c)$ is the hinge loss function.

In the next section we will develop an online learning framework for non-decomposable performance measures and instantiate our framework with the above mentioned loss functions $\ell_{\text{\preck}}$ and $\ell_{\text{pAUC}}$. Then in Section~\ref{sec:stochastic}, we will develop stochastic gradient methods for non-decomposable loss functions and prove error bounds for the same. There we will focus on a much larger family of loss functions including \preck, pAUC and F-measure.

\section{Online Learning with Non-decomposable Loss Functions}
We now present our online learning framework for non-decomposable loss functions. Traditional online learning takes place in several rounds, in each of which the player proposes some $\w_t\in \W$ while the adversary responds with a penalty function $\L_t : \W \rightarrow \R$ and a loss $\L_t(\w_t)$ is incurred. The goal is to minimize the \emph{regret} i.e. $\sum_{t=1}^T\L_t(\w_t) - \mathop{\arg\min}_{\w\in \W} \sum_{t=1}^T\L_t(\w)$. For point loss functions, the \emph{instantaneous} penalty $\L_t(\cdot)$ is encoded using a data point $(\vecx_t,y_t) \in \R^d \times \{-1, 1\}$ as $\L_t(\w) = \ell_\P(\x_t,y_t,\w)$. However, for (surrogates of) non-decomposable loss functions such as $\ell_{\text{pAUC}}$ and $\ell_{\text{\preck}}$ the definition of instantaneous penalty itself is not clear and remains a challenge.

To guide us in this process we turn to some properties of standard online learning frameworks. For point losses, we note that the best solution in hindsight is also the batch optimal solution. This is equivalent to the condition $\mathop{\arg\min}_{\w\in \W}\sum_{t=1}^T\L_t(\w) = \mathop{\arg\min}_{\w\in \W}\ell_{\P}(\vecx_{1:T}, y_{1:T}, \w)$ for non-decomposable losses. Also, since the batch optimal solution is agnostic to the ordering of points, we should expect $\sum_{t=1}^T\L_t(\w)$ to be invariant to permutations within the stream. By pruning away several naive definitions of $\L_t$ using these requirements, we arrive at the following definition:
\begin{equation}
	\label{eq:lt}
	\L_t(\w) = \ell_{\P}(\x_{1:t}, y_{1:t}, \w) - \ell_{\P}(\x_{1:(t-1)}, y_{1:(t-1)}, \w).
\end{equation} 

It turns out that the above is a very natural penalty function as it measures the amount of ``extra'' penalty incurred due to the inclusion of $\x_t$ into the set of points. It can be readily verified that $\sum_{t=1}^T \L_t(\w)=\ell_{\P}(\vecx_{1:T}, y_{1:T}, \w)$ as required. Also, this penalty function seamlessly generalizes online learning frameworks since for point losses, we have $\ell_{\P}(\x_{1:t},y_{1:t},\w) = \sum_{i=1}^t\ell_\P(\x_i,y_i,\w)$ and thus $\L_t(\w)=\ell_\P(\vecx_t, y_t, \w)$. We note that our framework also recovers the model for online AUC maximization used in \cite{KarSJK13} and \cite{oam-icml}. The notion of regret corresponding to this penalty is
\[
R(T)=\frac{1}{T}\sum_{t=1}^T \L_t(\w_t)-\mathop{\arg\min}_{\w\in \W}\frac{1}{T}\ell_\P(\vecx_{1:T}, y_{1:T}, \w).
\]
We note that $\L_t$, being the difference of two loss functions, is non-convex in general and thus, standard online convex programming regret bounds cannot be applied in our framework. Interestingly, as we show below, by exploiting structural properties of our penalty function, we can still get efficient low-regret learning algorithms, as well as online-to-batch conversion bounds in our framework.

\subsection{Low Regret Online Learning}
We propose an efficient Follow-the-Regularized-Leader (FTRL) style algorithm in our framework. Let $\vecw_1=\arg\min_{\vecw\in \W}\|\vecw\|_2^2$ and consider the following update:
\begin{align}
	\label{eq:ftrl} 
	\vecw_{t+1}&=\arg\min_{\vecw\in \W} \sum_{t=1}^t \L_t(\vecw)+\frac{\eta}{2}\|\vecw\|_2^2 = \arg\min_{\vecw\in \W}\ell_{\P}(\x_{1:t}, y_{t:t}, \vecw)+\frac{\eta}{2}\|\vecw\|_2^2\tag*{(FTRL)}
\end{align}

We would like to stress that despite the non-convexity of $\L_t$, the FTRL objective is strongly convex if $\ell_{\P}$ is convex and thus the update can be implemented efficiently by solving a regularized batch problem on $\x_{1:t}$. We now present our regret bound analysis for the FTRL update given above.
\begin{thm}
\label{thm:online}
Let $\ell_{\P}(\cdot, \vecw)$ be a convex loss function and $\W \subseteq \R^d$ be a convex set. Assume w.l.o.g. $\|\x_t\|_2\leq 1, \forall t$. Also, for the penalty function $\L_t$ in \eqref{eq:lt}, let $|\L_t(\vecw)-\L_t(\vecw')|\leq G_t\cdot\|\vecw-\vecw'\|_2, $ for all $t$ and all $\vecw, \vecw'\in \W$ for some $G_t > 0$. Suppose we use the update step given in \eqref{eq:ftrl} to obtain $\vecw_{t+1}, 0 \leq t \leq T-1$. Then for all $\w^\ast$, we have
\[
\frac{1}{T}\sum_{t=1}^T \L_t(\vecw_t)\leq \frac{1}{T}\ell_{\P}(\vecx_{1:T}, y_{1:T}, \vecw^\ast)+\norm{\w^\ast}_2\frac{\sqrt{2\sum_{t=1}^TG_t^2}}{T}.
\]
\end{thm}
See Appendix~\ref{app:proof-online} for a proof. The above result requires the penalty function $\L_t$ to be Lipschitz continuous i.e. be ``stable'' w.r.t. $\w$. Establishing this for point losses such as hinge loss is relatively straightforward. However, the same becomes non-trivial for non-decomposable loss functions as $\L_t$ is now the difference of two loss functions, both of which involve $\Om t$ data points. A naive argument would thus, only be able to show $G_t \leq O(t)$ which would yield vacuous regret bounds.

Instead, we now show that for the surrogate loss functions for \preck and pAUC, this Lipschitz continuity property does indeed hold. Our proofs crucially use a structural lemma given below that shows that sorted lists of inner products are Lipschitz at each fixed position. 
\begin{lem}[Structural Lemma]
\label{lem:list}
  Let $\vecx_1, \ldots, \vecx_t$ be $t$ points with $\|\vecx_i\|_2\leq 1$ $\forall t$. Let $\vecw, \vecw'\in \W$ be any two vectors. Let $z_i=\ip{\vecw}{\vecx_i}-c_i$ and $z_i'=\ip{\vecw'}{\vecx_i}-c_i$, where $c_i\in \R$ are constants independent of $\w,\w'$. Also, let $\{i_1, \dots, i_t\}$ and $\{j_1, \dots, j_t\}$ be ordering of indices such that $z_{i_1}\geq z_{i_2} \geq \dots \geq z_{i_t}$ and $z_{j_1}'\geq z_{j_2}' \geq \dots \geq z_{j_t}'$.
Then for any $1$-Lipschitz increasing function $g:\R\rightarrow\R$ (i.e. $\abs{g(u)-g(v)} \leq \abs{u-v}$ and $u \leq v \Leftrightarrow g(u) \leq g(v)$), we have, $\forall k$ $|g(z_{i_k})-g(z_{j_k}')|\leq 3\|\vecw-\vecw'\|_2$.
\end{lem}

See Appendix~\ref{app:proof-struct-lem} for a proof. Using this lemma we can show that the Lipschitz constant for $\ell_{\text{\preck}}$ is bounded by $G_t \leq 8$ which gives us a $\O{\sqrt{\frac{1}{T}}}$ regret bound for \preck (see Appendix~\ref{app:stab-preck} for the proof). In Appendix~\ref{sec:prbep-stab}, we show that the same technique can be used to prove a stability result for the structural SVM surrogate of the Precision-Recall Break Even Point (PRBEP) performance measure \cite{Joachims05} as well.
The case of pAUC is handled similarly. However, since pAUC discriminates between positives and negatives, our previous analysis cannot be applied directly. Nevertheless, we can obtain the following regret bound for pAUC (a proof will appear in the full version of the paper).
\begin{thm}
\label{thm:pauc}
Let $T_+$ and $T_-$ resp. be the number of positive and negative points in the stream and let $\vecw_{t+1}$, $0 \leq t \leq T-1$ be obtained using the FTRL algorithm \eqref{eq:ftrl}. Then we have
\[
\frac{1}{\beta T_+T_-}\sum_{t=1}^T \L_t(\vecw_t)\leq \min_{\w\in \W}\frac{1}{\beta T_+T_-}\ell_{\text{pAUC}}(\vecx_{1:T}, y_{1:T}, \vecw)+\O{\sqrt{\frac{1}{T_+}+\frac{1}{T_-}}}.
\]
\end{thm}
Notice that the above regret bound depends on both $T_+$ and $T_-$ and the regret becomes large even if one of them is small. This is actually quite intuitive because if, say $T_+=1$ and $T_-=T-1$, an adversary may wish to provide the lone positive point in the last round. Naturally the algorithm, having only seen negatives till now, would not be able to perform well and would incur a large error.

\subsection{Online-to-batch Conversion}
To present our bounds we generalize our framework slightly: we now consider the stream of $T$ points to be composed of $T/s$ batches $\Zb_1,\ldots,\Zb_{T/s}$ of size $s$ each. Thus, the instantaneous penalty is now defined as $\L_t(\w) = \ell_\P(\Zb_1,\ldots,\Zb_t,\w) - \ell_\P(\Zb_1,\ldots,\Zb_{t-1},\w)$ for $t = 1 \ldots T/s$ and the regret becomes $R(T,s) = \frac{1}{T}\sum_{t=1}^{T/s}\L_t(\w_t) - \mathop{\arg\min}_{\w\in \W}\frac{1}{T}\ell_\P(\vecx_{1:T}, y_{1:T}, \w)$. Let $\RR_{\P}$ denote the population risk for the (normalized) performance measure $\P$. Then we have:
\begin{thm}
\label{thm:otb-main}
Suppose the sequence of points $(\x_t,y_t)$ is generated i.i.d. and let $\w_1,\w_2,\ldots,\w_{T/s}$ be an ensemble of models generated by an online learning algorithm upon receiving these $T/s$ batches. Suppose the online learning algorithm has a guaranteed regret bound $R(T,s)$. Then for $\overline{\w} = \frac{1}{T/s} \sum_{t=1}^{T/s} \w_t$, any $\w^\ast \in \W$, $\epsilon \in (0, 0.5]$ and $\delta > 0$, with probability at least $1 - \delta$,
\[
\RR_\P(\overline\w) \,\leq\, (1+\epsilon)\RR_\P(\w^\ast) \,+\, R(T,s)  \,+\, e^{-\Om{s\epsilon^2}} \,+\,  \softO{\sqrt{\frac{s\ln(1/\delta)}{T}}}.
\]
In particular, setting $s =\tilde{\cal O}(\sqrt T)$ and $\epsilon = \sqrt[4]{\rfrac{1}{T}}$ gives us, with probability at least $1 - \delta$,
\[
\RR_\P(\overline\w) \,\leq\, \RR_\P(\w^\ast) \,+\, R(T,\sqrt{T}) \,+\, \softO{\sqrt[4]{\frac{\ln(1/\delta)}{T}}}.
\]
\end{thm}
We conclude by noting that for \preck and pAUC, $R(T,\sqrt{T}) \leq \O{\sqrt[4]{\rfrac{1}{T}}}$ (see Appendix~\ref{app:otb}).


\section{Stochastic Gradient Methods for Non-decomposable Losses}
\label{sec:stochastic}
The online learning algorithms discussed in the previous section present attractive guarantees in the sequential prediction model but are required to solve batch problems at each stage. This rapidly becomes infeasible for large scale data. To remedy this, we now present memory efficient stochastic gradient descent methods for batch learning with non-decomposable loss functions. The motivation for our approach comes from mini-batch methods used to make learning methods for \emph{point} loss functions amenable to distributed computing environments \cite{DekelG-BSX12,ZhangDW13}, we exploit these techniques to offer scalable algorithms for \emph{non-decomposable} loss functions.

\subsection{Single-pass Method with Mini-batches}
The method assumes access to a limited memory buffer and takes a pass over the data stream. The stream is partitioned into \emph{epochs}. In each epoch, the method accumulates points in the stream, uses them to form gradient estimates and takes descent steps. The buffer is flushed after each epoch. Algorithm~\ref{algo:spmb} describes the \spmb method. Gradient computations can be done using Danskin's theorem (see Appendix~\ref{app:exps-pauc}).

\begin{figure*}[t]
\hspace{-2.5ex}
\begin{minipage}[t]{0.5\linewidth}
\begin{algorithm}[H]
	\caption{\small \spmb: Single-Pass with Mini-batches}
	\label{algo:spmb}
	\begin{algorithmic}[1]
		\small{
			\REQUIRE Step length scale $\eta$, Buffer $B$ of size $s$
			\ENSURE A good predictor $\vecw \in \W$
			\STATE $\vecw_0 \leftarrow \veczero$, $B \leftarrow \phi$, $e \leftarrow 0$
			\WHILE{stream not exhausted}
				\STATE Collect $s$ data points $(\x^e_1,y^e_1),\ldots,(\x^e_s,y^e_s)$ in buffer $B$
				\STATE Set step length $\eta_e \leftarrow \frac{\eta}{\sqrt e}$
				\STATE $\vecw_{e+1} \leftarrow \Pi_\W\bs{\vecw_e + {\eta_e\nabla_\vecw\ell_\P(\x^e_{1:s},y^e_{1:s},\vecw_e)}}$\newline
				\mbox{}\hfill\COMMENT{$\Pi_\W$ projects onto the set $\W$}
				\STATE Flush buffer $B$
				\STATE $e \leftarrow e + 1$\COMMENT{start a new epoch}
			\ENDWHILE
			\STATE \textbf{return} {$\overline\vecw = \frac{1}{e}\sum_{i=1}^e\w_i$}
		}
	\end{algorithmic}
\end{algorithm}
\end{minipage}
\hspace{1ex}
\begin{minipage}[t]{0.52\linewidth}
\begin{algorithm}[H]
	\caption{\small \tpmb: Two-Passes with Mini-batches}
	\label{algo:tpmb}
	\begin{algorithmic}[1]
		\small{
			\REQUIRE Step length scale $\eta$, Buffers $B_+$, $B_-$ of size $s$
			\ENSURE A good predictor $\vecw \in \W$\newline
			\textbf{Pass 1:} $B_+ \leftarrow \phi$
			\STATE Collect random sample of pos. $\x^+_1,\ldots,\x^+_s$ in $B_+$\newline
			\textbf{Pass 2:} $\vecw_0 \leftarrow \veczero$, $B_- \leftarrow \phi$, $e \leftarrow 0$
			\WHILE{stream of negative points not exhausted}
				\STATE Collect $s$ negative points $\x^{e-}_1,\ldots,\x^{e-}_s$ in $B_-$
				\STATE Set step length $\eta_e \leftarrow \frac{\eta}{\sqrt e}$
				\STATE $\vecw_{e+1} \leftarrow \Pi_\W\bs{\vecw_e + {\eta_e\nabla_\vecw\ell_\P(\x^{e-}_{1:s},\x^+_{1:s},\vecw_e)}}$
				\STATE Flush buffer $B_-$
				\STATE $e \leftarrow e + 1$\COMMENT{start a new epoch}
			\ENDWHILE	
			\STATE \textbf{return} {$\overline\vecw = \frac{1}{e}\sum_{i=1}^e\w_i$}
		}
	\end{algorithmic}
\end{algorithm}
\end{minipage}
\end{figure*}

\subsection{Two-pass Method with Mini-batches}
The previous algorithm is unable to exploit relationships between data points across epochs which may help improve performance for loss functions such as pAUC. To remedy this, we observe that several real life learning scenarios exhibit mild to severe label imbalance (see Table~\ref{tab:dataset-stats} in Appendix~\ref{app:exps-pauc}) which makes it possible to store all or a large fraction of points of the rare label. Our two pass method exploits this by utilizing two passes over the data: the first pass collects all (or a random subset of) points of the rare label using some stream sampling technique \cite{KarSJK13}. The second pass then goes over the stream, restricted to the non-rare label points, and performs gradient updates. See Algorithm~\ref{algo:tpmb} for details of the \tpmb method.

\subsection{Error Bounds}
Given a set of $n$ labeled data points $(\x_i,y_i), i = 1 \ldots n$ and a performance measure $\P$, our goal is to approximate the empirical risk minimizer $\w^\ast = \underset{\w \in \W}{\arg\min}\ \ell_\P(\x_{1:n},y_{1:n},\w)$ as closely as possible. In this section we shall show that our methods \spmb and \tpmb provably converge to the empirical risk minimizer. We first introduce the notion of uniform convergence for a performance measure.
\begin{defn}
\label{def:uc}
We say that a loss function $\ell$ demonstrates uniform convergence with respect to a set of predictors $\W$ if for some $\alpha(s,\delta) = \text{poly}\br{\frac{1}{s},\log\frac{1}{\delta}}$, when given a set of $s$ points $\bx_1,\ldots,\bx_s$ chosen randomly from an arbitrary set of $n$ points $\bc{(\x_1,y_1),\ldots,(\x_n,y_n)}$ then w.p. at least $1 - \delta$, we have
\[
\underset{\w\in\W}{\sup}\abs{\ell_\P(\x_{1:n},y_{1:n},\w)-\ell_\P(\bx_{1:s},\by_{1:s},\w)} \leq \alpha(s,\delta).
\]
\end{defn}
Such uniform convergence results are fairly common for decomposable loss functions such as the squared loss, logistic loss etc. However, the same is not true for non-decomposable loss functions barring a few exceptions \cite{u-stat-rank,YeCLC12}. To bridge this gap, below we show that a large family of surrogate loss functions for popular non decomposable performance measures does indeed exhibit uniform convergence. Our proofs require novel techniques and do not follow from traditional proof progressions. However, we first show how we can use these results to arrive at an error bound.

\begin{thm}
\label{thm:erm}
Suppose the loss function $\ell$ is convex and demonstrates $\alpha(s,\delta)$-uniform convergence. Also suppose we have an arbitrary set of $n$ points which are randomly ordered, then the predictor $\overline\w$ returned by \spmb with buffer size $s$ satisfies w.p. $1 - \delta$,
\[
\ell_\P(\x_{1:n},y_{1:n},\overline\w) \leq \ell_\P(\x_{1:n},y_{1:n},\w_\ast) + 2\alpha\br{s,\frac{s\delta}{n}} + \O{\sqrt{\frac{s}{n}}}
\]
\end{thm}
We would like to stress that the above result does not assume i.i.d. data and works for arbitrary datasets so long as they are randomly ordered. We can show similar guarantees for the two pass method as well (see Appendix~\ref{app:proof-genbound-pauc}). Using regularized formulations, we can also exploit logarithmic regret guarantees \cite{log-regret}, offered by online gradient descent, to improve this result - however we do not explore those considerations here. Instead, we now look at specific instances of loss functions that posses the desired uniform convergence properties. As mentioned before, due to the combinatorial nature of these performance measures, our proofs do not follow from traditional methods.

\begin{thm}[Partial Area under the ROC Curve]
\label{thm:gen-bound-rank}
For any convex, monotone, Lipschitz, classification surrogate $\phi : \R \rightarrow \R_+$, the surrogate loss function for the $(0,\beta)$-partial AUC performance measure defined as follows exhibits uniform convergence at the rate $\alpha(s,\delta) = \O{\sqrt{{\log(1/\delta)}/{s}}}$:
\[
\frac{1}{\lceil\beta n_-\rceil n_+}\sum_{i:y_i>0}\sum_{j:y_j<0}\mathbb{T}^-_{\beta, t}(\vecx_j, \vecw)\cdot\phi(\vecx_i^\top\vecw- \vecx_j^\top\vecw)
\]
\end{thm}
See Appendix~\ref{app:uc-proof-pauc} for a proof sketch. This result covers a large family of surrogate loss functions such as hinge loss \eqref{eq:pauc_cvx}, logistic loss etc. Note that the insistence on including only top ranked negative points introduces a high degree of non-decomposability into the loss function. A similar result for the special case $\beta=1$ is due to \cite{u-stat-rank}. We extend the same to the more challenging case of $\beta < 1$.

\begin{thm}[Structural SVM loss for \preck]
\label{thm:gen-bound-struct}
The structural SVM surrogate for the \preck performance measure (see \eqref{eq:prec}) exhibits uniform convergence at the rate $\alpha(s,\delta) = \O{\sqrt{{\log(1/\delta)}/{s}}}$.
\end{thm}
We defer the proof to the full version of the paper. The above result can be extended to a large family of performances measures introduced in \cite{Joachims05} that have been widely adopted \cite{YeCLC12,DaskalakiKA06,KubatM97} such as F-measure, G-mean, and PRBEP. The above indicates that our methods are expected to output models that closely approach the empirical risk minimizer for a wide variety of performance measures. In the next section we verify that this is indeed the case for several real life and benchmark datasets.


\section{Experimental Results}
\label{sec:exps}
\begin{figure*}[t]
               \centering
               \subfigure[PPI]{
                              \includegraphics[scale=0.525]{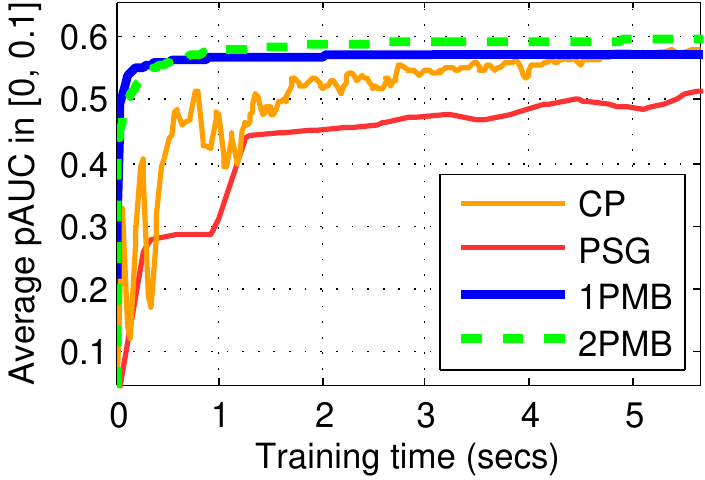}
                              \label{subfig:ppi-pauc}
               }
               \subfigure[KDDCup08]{
                              \includegraphics[scale=0.525]{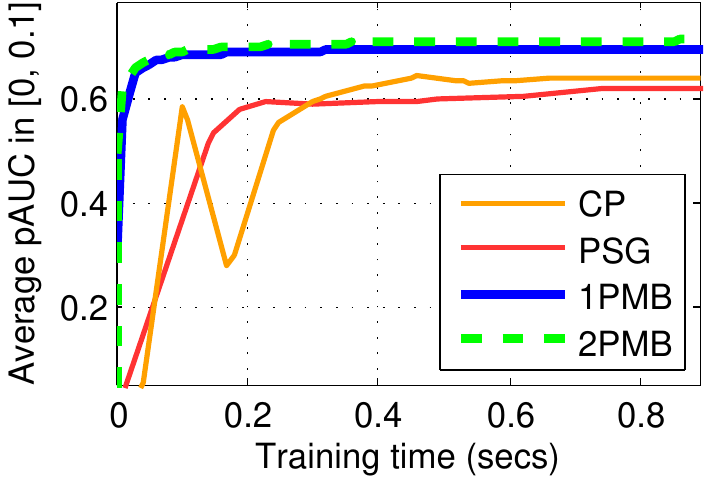}
                              \label{subfig:kdd08-pauc}
               }               
               \subfigure[IJCNN]{
               				
                              \includegraphics[scale=0.525]{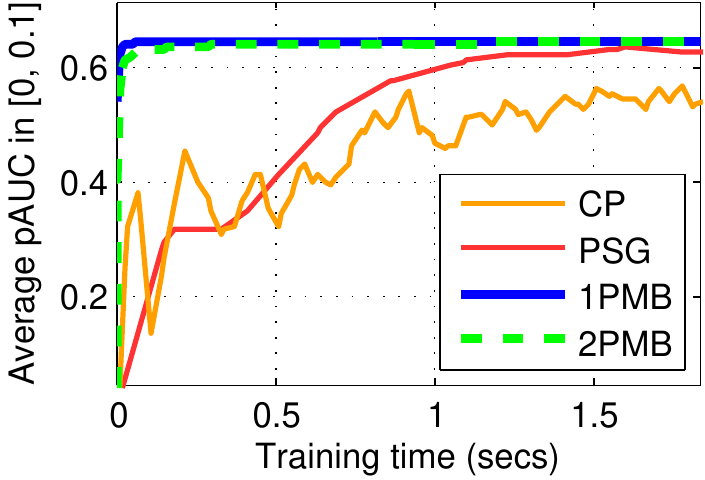}
                              \label{subfig:ijcnn1-pauc}
               }          
               \subfigure[Letter]{
                              \includegraphics[scale=0.525]{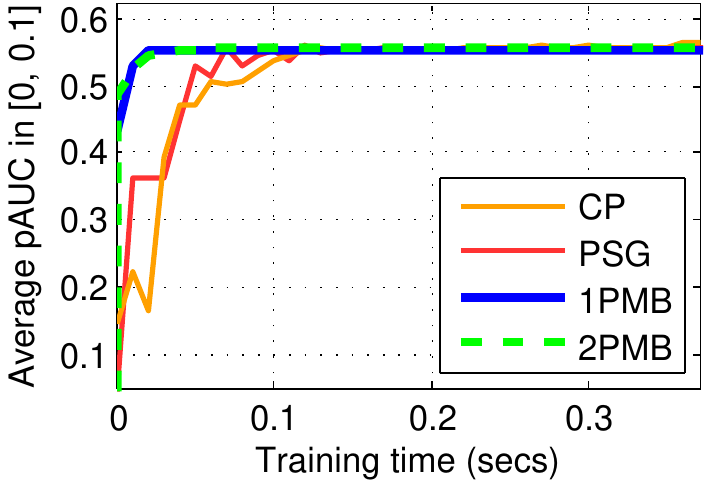}
                              \label{subfig:letter-pauc}
               }   
               \vspace*{-2ex}
               \caption{Comparison of stochastic gradient methods with the cutting plane (CP) and projected subgradient (PSG) methods on partial AUC maximization tasks. The epoch lengths/buffer sizes for \spmb and \tpmb were set to 500.
               \vspace*{-1ex}
				}
               \label{fig:pauc-expts}
\end{figure*}

\begin{figure*}[t]
               \centering
               \subfigure[PPI]{
                              \includegraphics[scale=0.525]{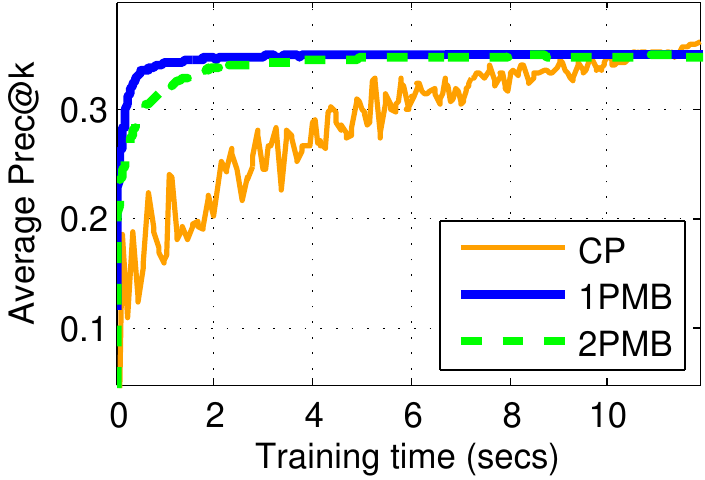}
                              \label{subfig:ppi-preck}
               }
               \subfigure[KDDCup08]{
                              \includegraphics[scale=0.525]{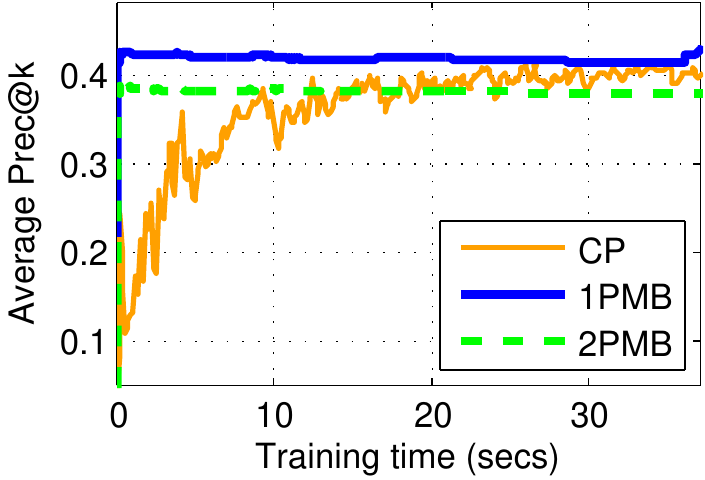}
                              \label{subfig:kdd08-preck}
               }               
               \subfigure[IJCNN]{
               				
                              \includegraphics[scale=0.525]{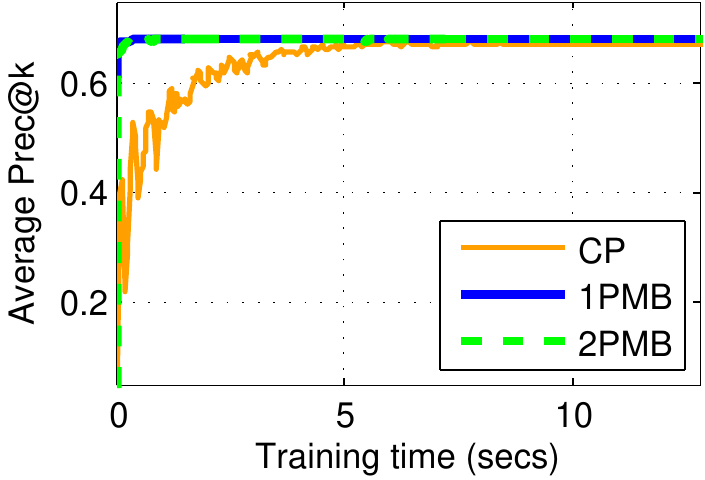}
                              \label{subfig:ijcnn1-preck}
               }              
               \subfigure[Letter]{
                              \includegraphics[scale=0.525]{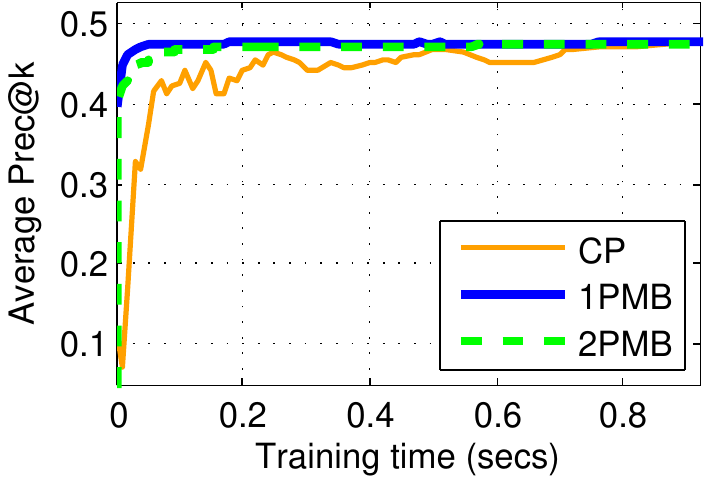}
                              \label{subfig:letter-preck}
               }  
               \vspace*{-2ex}
               \caption{Comparison of stochastic gradient methods with cutting plane (CP) methods on \preck maximization tasks. The epoch lengths/buffer sizes for \spmb and \tpmb were set to 500.
			}
               \label{fig:prbep-expts}
\end{figure*}

We evaluate the proposed stochastic gradient methods on several real-world and benchmark datasets.\\

{\bf Performance measures}: We consider three measures, 1) partial AUC in the false positive range $[0, 0.1]$, 2) Prec$@k$ with $k$ set to the proportion of positives (PRBEP), and 3) F-measure.\\

{\bf Algorithms}: For partial AUC, we compare against the state-of-the-art cutting plane (CP) and projected subgradient methods (PSG) proposed in \cite{NarasimhanA13b}; unlike the (online) stochastic methods considered in this work, the PSG method is a `batch' algorithm which, at each iteration, computes a subgradient-based update over the entire training set. For \preck and F-measure, we compare our methods against cutting plane methods from \cite{Joachims05}. We used structural SVM surrogates for all the measures.\\

{\bf Datasets}: We used several data sets for our experiments (see Table~\ref{tab:dataset-stats}); of these, KDDCup08 is from the KDD Cup 2008 challenge and involves a breast cancer detection task \cite{kddcupRYK08}, PPI contains data for a protein-protein interaction prediction task \cite{ppiQBK06}, and the remaining datasets are taken from the UCI repository \cite{uci}.\\

{\bf Parameters}: We used 70\% of the data set for training and the remaining for testing, with the results averaged over 5 random train-test splits. Tunable parameters such as step length scale were chosen using a small validation set. All experiments used a buffer of size $500$. Epoch lengths were set equal to the buffer size. Since a single iteration of the proposed stochastic methods is very fast in practice, we performed multiple passes over the training data (see Appendix~\ref{app:exps-pauc} for details).\\


\begin{table}
\centering
\begin{tabular}{|c|c|c|c|}
\hline
\textbf{Dataset}		&\textbf{Data Points}	& \textbf{Features}	& \textbf{Positives}\\\hline
KDDCup08	& 102,294		& 117		& 0.61\%\\\hline
PPI			& 240,249		& 85		& 1.19\%\\\hline
Letter		& 20,000		& 16		& 3.92\%\\\hline
IJCNN		& 141,691		& 22		& 9.57\%\\\hline
\end{tabular}
\caption{Statistics of datasets used.}
\label{tab:dataset-stats}
\end{table}

\begin{figure*}[t]
	\begin{minipage}[c]{0.55\linewidth}
		\centering
		\small
		\begin{table}[H]
			\begin{tabular}{|c|c|c|c|c|}
				\hline
				\textbf{Measure}	& \textbf{1PMB}	& \textbf{2PMB}	& \textbf{CP} \\\hline
				pAUC		& 0.10 (68.2)	& 0.15 (\bfseries 69.6)		& \,\,\,0.39 (62.5) \\\hline
				\preck	& 0.49 (\bfseries 42.7)	& 0.55 (38.7)		& 23.25 (40.8)	\\\hline
			\end{tabular}
			\caption{Comparison of training time (secs) and accuracies (in brackets) of \spmb, \tpmb and cutting plane methods for pAUC (in $[0, 0.1]$) and \preck maximization tasks on the KDD Cup 2008 dataset.}
			\label{tab:betas}
		\end{table}
	\end{minipage}
	\hspace{0.2cm}
	\begin{minipage}[t]{0.4\linewidth}
		\vspace*{-12ex}
	   \centering
					  \includegraphics[scale=0.525]{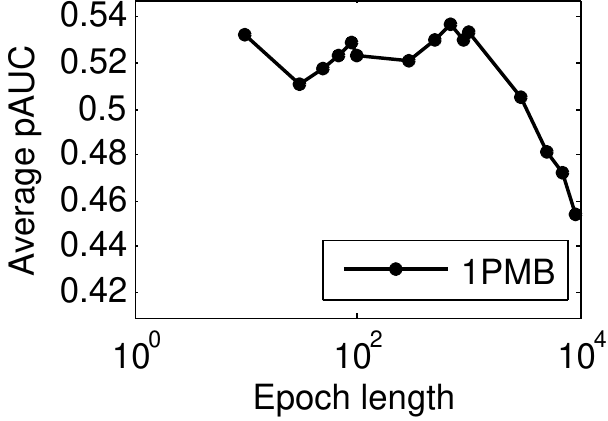}
					  \includegraphics[scale=0.525]{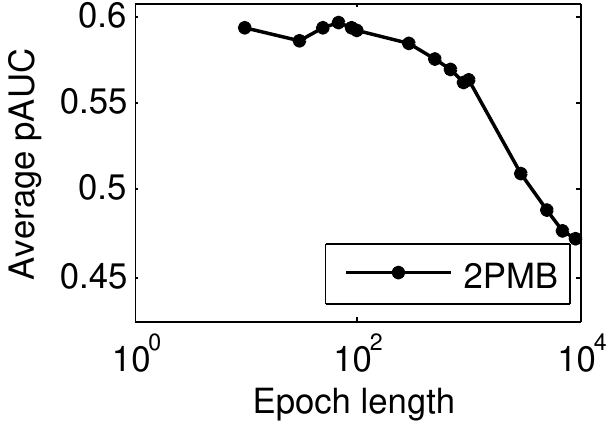}
	   \vspace{-15pt}
	   \caption{Performance of \spmb and \tpmb on the PPI dataset with varying epoch/buffer sizes for pAUC tasks.}
	   \label{fig:epoch}
	\end{minipage}
\end{figure*}

{\bf Results}: The results for pAUC and \preck maximization tasks are shown in the Figures~\ref{fig:pauc-expts} and \ref{fig:prbep-expts}. We found the proposed stochastic gradient methods to be several orders of magnitude faster than the baseline methods, all the while achieving comparable or better accuracies. For example, for the pAUC task on the KDD Cup 2008 dataset, the \textbf{1PMB} method achieved an accuracy of 64.81\% within 0.03 seconds, while even after 0.39 seconds, the cutting plane method could only achieve an accuracy of 62.52\% (see Table~\ref{tab:betas}). As expected, the (online) stochastic gradient methods were faster than the `batch' projected subgradient descent method for pAUC as well.
We found similar trends on \preck (see Figure~\ref{fig:prbep-expts}) and F-measure maximization tasks as well. For F-measure tasks, on the KDD Cup 2008 dataset, for example, the \spmb method achieved an accuracy of 35.92 within 12 seconds whereas, even after 150 seconds, the cutting plane method could only achieve an accuracy of 35.25.

The proposed stochastic methods were also found to be robust to changes in epoch lengths (buffer sizes) till such a point where excessively long epochs would cause the number of updates as well as accuracy to dip (see Figure~\ref{fig:epoch}). The \textbf{2PMB} method was found to offer higher accuracies for pAUC maximization on several datasets (see Table \ref{tab:betas} and Figure~\ref{fig:pauc-expts}), as well as be more robust to changes in buffer size (Figure~\ref{fig:epoch}).
We defer results on more datasets and performance measures to the full version of the paper.

The cutting plane methods were generally found to exhibit a zig-zag behaviour in performance across iterates. This is because these methods solve the dual optimization problem for a given performance measure; hence the intermediate models do not necessarily yield good accuracies. On the other hand, (stochastic) gradient based methods directly offer progress in terms of the primal optimization problem, and hence provide good intermediate solutions as well. This can be advantageous in scenarios with a time budget in the training phase.

\section*{Acknowledgements}
The authors thank Shivani Agarwal for helpful comments. They also thank the anonymous reviewers for their suggestions. HN thanks support from a Google India PhD Fellowship.


\bibliographystyle{unsrt}
\bibliography{ref}



\appendix


\allowdisplaybreaks


\section{Proof of Theorem~\ref{thm:online}}
\label{app:proof-online}
Broadly, we follow the proof structure of FTRL given in \cite{Rakhlin09, SahaJT12}. We first observe that the ``forward regret'' analysis follows easily despite the non-convexity of $\L_t$. That is, 
\begin{equation}
\label{eq:fr}
\sum_{t=1}^T \L_t(\vecw_{t+1})\leq  \vecx_{1:T}, y_{1:T}, \vecw_*) + \frac{\eta}{2}\|\vecw_*\|_2^2,
\end{equation}
where $\vecw_*=\arg\min_{\vecw\in \W} \vecx_{1:T}, y_{1:T}, \vecw)$. The proof of this statement can be found in \cite[Theorem 7]{SahaJT12} and is reproduced below as Lemma~\ref{lem:fwd-regret-proof} for completeness. Next, using strong convexity of the regularizer $\|\vecw\|_2^2$ and optimality of $\vecw_t$ and $\vecw_{t+1}$ for their respective update steps, we get:
\begin{eqnarray*}
\ell_\P(\x_{1:t},y_{1:t},\w_{t+1}) + \frac{\eta}{2} \|\vecw_{t+1}-\vecw_t\|_2^2 &\leq& \ell_\P(\x_{1:t},y_{1:t},\w_t)\\
\ell_\P(\x_{1:t-1},y_{1:t-1},\w_{t+1}) &\geq& \ell_\P(\x_{1:t-1},y_{1:t-1},\w_{t}) + \frac{\eta}{2} \|\vecw_{t+1}-\vecw_t\|_2^2,
\end{eqnarray*}
which when subtracted, give us
\begin{equation}
\label{eq:stability}
\eta \|\vecw_{t+1}-\vecw_t\|_2^2 \leq \L_t(\vecw_t)-\L_{t}(\vecw_{t+1})\leq G_t\|\vecw_{t+1}-\vecw_t\|_2,
\end{equation}
where the last inequality follows using the Lipschitz continuity of $\L_t$. We now use the fact that
\[
\sum_{t=1}^T \L_t(\vecw_t)=\sum_{t=1}^T \L_{t}(\vecw_{t+1})+\sum_{t=1}^T (\L_t(\vecw_t)- \L_{t}(\vecw_{t+1})),
\]
along with \eqref{eq:fr} and \eqref{eq:stability} to get
\[
\sum_{t=1}^T \L_t(\vecw_t) \leq \vecx_{1:T}, y_{1:T}, \vecw_*) + \frac{\eta}{2}\|\vecw_*\|_2^2 + \frac{\sum_{t=1}^TG_t^2}{\eta}.
\]
The result now follows by selecting $\eta=\sqrt{2\sum_{t=1}^TG_t^2/\norm{\w_\ast}_2^2}$.
\begin{lem}
\label{lem:fwd-regret-proof}
For the setting described in Theorem~\ref{thm:online}, we have
\[
\sum_{t=1}^T \L_t(\vecw_{t+1}) \leq \vecx_{1:T}, y_{1:T}, \vecw_*) + \frac{\eta}{2}\|\vecw_*\|_2^2
\]
\end{lem}
\begin{proof}
Let $\L_0(\w) := \frac{\eta}{2}\norm{\w}_2^2$. Thus, we can equivalently write the FTRL update in \ref{eq:ftrl} as
\[
\vecw_{t+1} = \arg\min_{\vecw\in \W} \sum_{\tau=0}^t \L_\tau(\vecw).
\]
Now, using the optimality of $\w_{t+1}$ at time $t$, we get
\begin{align}
\sum_{\tau=0}^t \L_\tau(\vecw_{t+1}) \leq \sum_{\tau=0}^t \L_\tau(\w_\ast)
\end{align}
Combining this with the optimality of $\w_t$ at time $t-1$, we get
\begin{align}
\sum_{\tau=0}^{t-1} \L_\tau(\vecw_t) + \L_t(\vecw_{t+1}) \leq \sum_{\tau=0}^t \L_\tau(\vecw_{t+1}) \leq \sum_{\tau=0}^t \L_\tau(\w_\ast)
\end{align}
Repeating this argument gives us
\[
\sum_{\tau=0}^{t} \L_\tau(\vecw_{\tau+1}) \leq \sum_{\tau=0}^t \L_\tau(\w_\ast),
\]
which proves the result.
\end{proof}

\section{Proof of Lemma~\ref{lem:list}}
\label{app:proof-struct-lem}
We consider the following four exhaustive cases in turn:
\begin{enumerate}[\bfseries{Case} 1.]
	\item{$z_{i_k}\geq z_{j_k}$ and $z_{j_k}'\geq z_{i_k}'$}\\
		We have the following set of inequalities
		\begin{eqnarray*}
			g(z_{i_k}) &=& g(\ip{\w}{\x_{i_k}}-c_i)\\
					&\leq& g(\ip{\w'}{\x_{i_k}}-c_i) + \abs{\ip{\w-\w'}{\x_{i_k}}}\\
					&\leq& g(\ip{\w'}{\x_{i_k}}-c_i) + \norm{\w-\w'}_2\\
					&=& g(z'_{i_k}) + \norm{\w-\w'}_2\\
					&\leq& g(z'_{j_k}) + \norm{\w-\w'}_2,
		\end{eqnarray*}
		where the first inequality follows by the Lipschitz assumption, the second follows by Cauchy-Schwartz inequality and the last follows by the case assumption $z_{j_k}'\geq z_{i_k}'$ and the fact that $g$ is an increasing function. By renaming $i \leftrightarrow j$ and $\w \leftrightarrow \w'$, we also have $g(z'_{j_k}) \leq g(z_{i_k}) + \norm{\w-\w'}_2$. This establishes the result for the specific case.
	\item{$z_{i_k}\leq z_{j_k}$ and $z_{j_k}'\leq z_{i_k}'$}\\
		This case follows similar to the case above.
	\item{$z_{i_k}\geq z_{j_k}$ and $z_{j_k}'\leq z_{i_k}'$}\\
		Using the above conditions $z_{j_k}$ does not belong to the top $k$ elements of $z_1, \dots, z_t$, but both $z'_{i_k}$ and $z'_{j_k}$ belong to the top $k$ elements of $z_1', \dots, z_t'$. Using the pigeonhole principle, there exists an index $s$ such that $z_s\geq z_{i_k}$ but $z_s\leq z'_{j_k}$. Hence, using arguments similar to Case 1, we get the following two bounds: 
\begin{align*}
|g(z'_{i_k})-g(z_s)|&\leq \|\vecw-\vecw'\|_2,\\
|g(z_s)-g(z'_{j_k})|&\leq \|\vecw-\vecw'\|_2.
\end{align*}
We also have $\abs{g(z'_{i_k}) - g(z_{i_k})} \leq \abs{\ip{\w-\w'}{\x_{i_k}}} \leq \norm{\w-\w'}_2$. Adding these three inequalities gives us the desired result.
	\item{$z_{i_k}\leq z_{j_k}$ and $z_{j_k}'\geq z_{i_k}'$}\\
		This case follows similar to the case above. \qedhere
\end{enumerate}
These cases are exhaustive and we thus conclude the proof.


\section{Stability result for \preck}
\label{app:stab-preck}
\begin{lem}
\label{lem:prec}
Let $\ell_{\text{\preck}}$ be the surrogate for \preck as defined in \eqref{eq:prec}, $\|\x_t\|_2\leq 1, \forall t$ and $\L_t$ be defined as in \eqref{eq:lt}. Then $\forall\w,\w'\in\W$, $|\L_t(\vecw)-\L_t(\vecw')|\leq 8\|\vecw-\vecw'\|_2$.
\end{lem}
\begin{proof}
Recall that, the loss function corresponding to \preck is defined as: 
\begin{align}
  \label{eq:prec_cvx1}
  \ell_{\text{\preck}}(\x_{1:t}, y_{1:t}, \vecw) &= \max_{\substack{q\in \{-1, 1\}^t\\\sum_i (q_i+1)=2\lceil k t\rceil}} \sum_{i=1}^t (q_i-y_i) \vecx_i^T\vecw - \sum_{i=1}^t q_i y_i \\
&= \underbrace{\max_{\substack{q\in \{-1, 1\}^t\\\sum_i (q_i+1)=2\lceil k t\rceil}} \sum_{i=1}^t q_i \vecx_i^T\vecw - \sum_{i=1}^t q_i y_i}_{A(\x_{1:t},y_{1:t},\w)} - \underbrace{\sum_{i=1}^ty_i\vecx_i^T\vecw}_{B(\x_{1:t},y_{1:t},\w)}
\end{align}
Since $B(\x_{1:t},y_{1:t},\w)$ is a decomposable loss function, it can at most add a constant (because of the assumptions made by us, that constant can be shown to be no bigger than $1$) to the Lipschitz constant of $\L_t$. Hence we concentrate on bounding the contribution of $A(\x_{1:t},y_{1:t},\w)$ to the Lipschitz constant of $\L_t$. Define $z_i=\ip{\vecw}{\vecx_i}-y_i$ and $z_i'=\ip{\vecw'}{\vecx_i}-y_i$. It will be useful to rewrite $A(\x_{1:t},y_{1:t},\w)$ as follows (and drop mentioning the dependence on $\x_{1:t}$ for notational simplicity): 
\begin{equation}
  \label{eq:prec1}
  p_t(\vecw) = 2\max_{\substack{q\in \{1, 0\}^t\\\sum_i q_i=\lceil k t\rceil}} \sum_{i=1}^t q_i z_i-\sum_{i=1}^t z_i. 
\end{equation}
Similarly, we can define $p_{t-1}(\vecw)$ as well. Now we have
\begin{eqnarray*}
\L_t(\w) - \L_t(\w') &=& p_t(\vecw)-p_{t-1}(\vecw)-p_t(\vecw')+p_{t-1}(\vecw') + y_t\x_t(\w'-\w)\\
					&\leq& \underbrace{p_t(\vecw)-p_{t-1}(\vecw)-p_t(\vecw')+p_{t-1}(\vecw')}_{\Delta_t(\w,\w')} + \norm{\w-\w'}_2
\end{eqnarray*}
Our mail goal in the sequel will be to show that $\Delta_t(\w,\w')\leq\O{\norm{\w-\w'}_2}$ which shall establish the desired Lipschitz continuity result.
Now for both vectors $\w,\w'$ and time instances $t-1,t$, let us denote the optimal assignments as follows:
\begin{align*}
  a^t&=\underset{\substack{q\in \{1, 0\}^t\\\sum_i q_i=\lceil k t\rceil}}{\arg\max} \sum_{i=1}^t q_i z_i,\qquad\qquad\qquad &b^t&=\underset{\substack{q\in \{1, 0\}^t\\\sum_i q_i=\lceil k t\rceil}}{\arg\max} \sum_{i=1}^t q_i z_i',\\
  a^{t-1}&=\underset{\substack{q\in \{1, 0\}^{(t-1)}\\\sum_i q_i=\lceil k (t-1)\rceil}}{\arg\max} \sum_{i=1}^{t-1} q_i z_i,\quad\quad\quad &b^{t-1}&=\underset{\substack{q\in \{1, 0\}^{(t-1)}\\\sum_i q_i=\lceil k (t-1)\rceil}}{\arg\max} \sum_{i=1}^{t-1} q_i z_i'.
\end{align*}
Also, define indices $1\leq i_r \leq t-1$ and $1\leq j_s\leq t-1$ as: 
\begin{align*}
	&z_{i_1}\geq z_{i_2}\dots \geq z_{i_{t-1}},\\
	&z'_{j_1}\geq z'_{j_2}\dots \geq z'_{j_{t-1}}.
\end{align*}
Now, note that \eqref{eq:prec1} involves maximization of a linear function, hence the optimizing assignment $q$ will always lie on the boundary of the Boolean hypercube with the cardinality constraint. Hence, $a^t$ can be obtained by setting $a^t_{i_r}=1,\ \forall 1\leq r\leq \lceil k t\rceil$ and $a^t_{i_r}=0,\ \forall r>\lceil k t\rceil$, similarly for $b_t$. We consider the following two cases and within each, four subcases which establish the result.

In the rest of the proof, all invocations of Lemma~\ref{lem:list} shall use the identity function for $g(\cdot)$ and $c_i = y_i$. Clearly this satisfies the prerequisites of Lemma~\ref{lem:list} since the identity function is $1$-Lipschitz and increasing.

\renewcommand{\labelenumii}{\bfseries{Case} \arabic{enumi}.\arabic{enumii}}

\begin{enumerate}[\bfseries{Case} 1]
	\item $\lceil k t\rceil=\lceil k (t-1)\rceil=\alpha$\\
	Within this, we have the following four exhaustive subcases:
	\begin{enumerate}
		\item $z_t\leq z_{i_\alpha}$ and $z_t'\leq z_{j_\alpha}'$\\
		The above condition implies that both $a^t_t=0$ and $b^t_t=0$. Furthermore, $a^t_{1:(t-1)}=a^{t-1}$ and $b^t_{1:(t-1)}=b^{t-1}$. As a result we have
		\[
			\Delta_t(\w,\w')=-z_t+z_t'=-\ip{\vecw}{\vecx_t}+\ip{\vecw'}{\vecx_t}\leq \|\vecw-\vecw'\|_2.
		\]
		\item $z_t> z_{i_\alpha}$ and $z_t'\leq z_{j_\alpha}'$\\
		The above condition implies that $a^t_t=1$ and $b^t_t=0$. Hence,  $b^t_{1:(t-1)}=b^{t-1}$. Also, as $a^t_t$ is turned on, the cardinality constraint dictates that one previously positive index should be turned off. That is,  $a^t_{i_\alpha}=0$, but $a^{t-1}_{i_\alpha}=1$. Finally, $a^t_{i_r}=a^{t-1}_{i_r}, r\neq \alpha\ \text{ and }\ r < t$. Using the above observations, we have the following sequence of inequalities:
		\begin{eqnarray*}
			\Delta_t(\w,\w') &=& (2(z_t-z_{i_\alpha}) - z_t) - (0 - z_t')\\
							&=& (z_t - z_{i_\alpha}) + (z_t' - z_{i_\alpha})\\
							&=& (z_t - z_t') + 2(z_t' - z_{i_\alpha})\\
							&\leq& (z_t - z_t') + 2(z_{j_\alpha}' - z_{i_\alpha})\\
							&\leq& 7\norm{\w-\w'}_2,
		\end{eqnarray*}
		where the third inequality follows from the case assumptions and the final inequality follows from an application of Cauchy Schwartz inequality and Lemma~\ref{lem:list}.
		\item $z_t\leq z_{i_\alpha}$ and $z_t'> z_{j_\alpha}'$\\
		In this case, we can analyze similarly to get
		\begin{eqnarray*}
			\Delta_t(\w,\w') &=& (0 - z_t) - (2(z_t'-z'_{j_\alpha}) - z_t')\\
							&=& (z'_{j_\alpha} - z_t) + (z'_{j_\alpha} - z_t')\\
							&=& (z_t' - z_t) + 2(z'_{j_\alpha} - z_t')\\
							&\leq& (z_t' - z_t)\\
							&\leq& 3\norm{\w-\w'}_2.
		\end{eqnarray*}
		\item $z_t> z_{i_\alpha}$ and $z_t'> z_{j_\alpha}'$\\
		In this case, both $a^t_t=1$ and $b^t_t=1$. Hence, both $a^t_{i_\alpha}=0$ and $b^t_{j_\alpha}=0$. The remaining terms of $a^t$ and $a^{t-1}$ (similarly for $b^t$ and $b^{t-1}$) remain the same. That is, we have
		\begin{eqnarray*}
			\Delta_t(\w,\w') &=& (2(z_t-z_{i_\alpha})-z_t) - (2(z_t'-z'_{j_\alpha})-z_t')\\
							&=& (z_t - z_t') - 2(z_{i_\alpha} - z'_{j_\alpha})\\
							&\leq& 7\norm{\w-\w'}_2.
		\end{eqnarray*}
	\end{enumerate}
	\item $\lceil k t\rceil=\lceil k (t-1)\rceil+1=\alpha$\\
	Here again, we consider the following four exhaustive subcases:
	\begin{enumerate}
		\item $z_t\leq z_{i_\alpha}$ and $z_t'\leq z_{j_\alpha}'$\\
		The above condition implies that $a^t_t=0$ and $b^t_t=0$. Also, one new positive is included in both $a^t$ and $b^t$, i.e., $a^t_{i_\alpha}=1$ and $b^t_{j_\alpha}=1$. The remaining entries of $a^t$ and $b^t$ remains the same. Hence,
		\[
			\Delta_t(\w,\w')=(2z_{i_\alpha}-z_t)-(2z'_{j_\alpha}-z'_t) = 2(z_{i_\alpha} - z'_{j_\alpha})-(z_t-z_t') \leq 9\norm{\w-\w'}_2.
		\]
		\item $z_t> z_{i_\alpha}$ and $z_t'\leq z_{j_\alpha}'$\\
		The above condition implies that $a^t_t=1$ and $b^t_t=0$. Also, $b^t_{j_\alpha}=1$. The remaining entries of $a^t$ and $b^t$ remains the same. Hence we have
		\begin{eqnarray*}
			\Delta_t(\w,\w') &=& (2z_t-z_t)-(2z'_{j_\alpha}-z'_t)\\
							&=& (z_t - z'_{j_\alpha}) + (z_t' - z'_{j_\alpha})\\
							&=& (z_t - z'_t) + 2(z_t' - z'_{j_\alpha})\\
							&\leq& 3\norm{\w-\w'}_2.
		\end{eqnarray*}
		\item $z_t\leq z_{i_\alpha}$ and $z_t'> z_{j_\alpha}'$\\
		In this case we have
		\begin{eqnarray*}
			\Delta_t(\w,\w') &=& (2z_{i_\alpha}-z_t)-(2z'_t-z'_t)\\
							&=& (z_{i_\alpha} - z_t) + (z_{i_\alpha} - z_t')\\
							&=& (z'_t - z_t) + 2(z_{i_\alpha} - z_t')\\
							&\leq& (z'_t - z_t) + 2(z_{i_\alpha} - z_{j_\alpha}')\\
							&\leq& 7\norm{\w-\w'}_2.
		\end{eqnarray*}
		\item $z_t> z_{i_\alpha}$ and $z_t'> z_{j_\alpha}'$\\
		The above condition implies that $a^t_t=1$ and $b^t_t=1$. The remaining entries of $a^t$ and $b^t$ remains the same. Hence,
		\[
			\Delta_t(\w,\w')=(2z_t-z_t)-(2z'_{t}-z'_t)=z_t-z_t'\leq 3\norm{\w-\w'}_2.
		\]
	\end{enumerate}
\end{enumerate}
Taking the worst case Lipschitz constants from these 8 subcases and adding the contribution of $B(\x_{1:t},y_{1:t},\w)$ concludes the proof.
\end{proof}

\section{Extension to Precision-Recall Break Even Point (PRBEP)}
\label{sec:prbep-stab}
We note that the above discussion can easily be extended to prove stability results for the structural surrogate loss for the PRBEP performance measure \cite{Joachims05}. Recall that the PRBEP measure essentially measures the precision (equivalently recall) of a predictor when thresholded at a point that equates the precision and recall. Since we have $\text{Prec} = \frac{\text{TP}}{\text{TP}+\text{FP}}$ and $\text{Rec} = \frac{\text{TP}}{\text{TP}+\text{FN}}$, the break even point is reached at a threshold where $\text{TP}+\text{FP} = \text{TP}+\text{FN}$. Notice that the left hand side equals the number of points that are predicted as positive whereas the right hand side equals the number of points that are actual positives.

Thus, the PRBEP is achieved at a threshold that predicts as many points as positive as there are actual positives which gives us the formal definition of this performance measure
\begin{equation}
\text{PRBEP}(\w) := \sum_{j:\T_{\br{\frac{t_+}{t},t}}(\x_j,\w) = 1}\Ind{y_j = 1}.
\end{equation}
Note that this is equivalent to the definition of \preck with $k = \frac{t_+}{t}$. Correspondingly, we can also define the structural SVM surrogate for this performance measure as
\begin{equation}
  \label{eq:prbep}
  \ell_{\text{PRBEP}}(\vecw) = \max_{\substack{\bar{\vecy}\in \{-1, +1\}^t\\\sum_i (\bar{y}_i+1)=2t_+}} \sum_{i=1}^t (\bar{y}_i-y_i) \vecx_i^T\vecw - \sum_{i=1}^t y_i\bar{y}_i.
\end{equation}

Given this, it is easy to see that the proof of Lemma~\ref{lem:prec} would apply to this case as well. The only difference in applying the analysis would be that Case 1 and its subcases would apply when $y_t < 0$ which is when the incoming point is negative and hence the number of actual positives in the stream does not go up. Case 2 and its subcases would apply when $y_t > 0$ in which case the number of points to be considered while calculating precision would have to be increased by 1.


\newcommand{\termm}{\textrm{term}}
\newcommand{\Pff}{\widetilde{\mathcal{R}}}
\newcommand{\paucc}{\textrm{pAUC}}
\newcommand{\Xbb}{\vec{Z}}
\newcommand{\exxL}{\widetilde{\L}}
\newcommand{\hell}{\widehat{\ell}}
\newcommand{\Expp}[2]{{{\mathbb E}}_{#1}\bsd{{#2}}}
\newcommand{\Probb}{{\mathbb P}}

\section{Online-to-batch Conversion}
\label{app:otb}
This section presents a proof of the regret bound in the batch model considered in Theorem~\ref{thm:otb-main} and a proof sketch of the online-to-batch conversion result. The full proof shall appear in the full version of the paper. We will consider in this section, the pAUC measure in the \tpmb setting wherein positives are assumed to reside in the buffer and negatives are streaming in. The case of the \preck measure in the usual \spmb setting can be handled similarly. Additionally, we will show in Appendix~\ref{app:uc-proof-pauc} that for the case of pAUC, the contributions from a large enough buffer of randomly chosen positive points mimics the contributions of the entire population of positive points. Thus, for pAUC, it suffices to show the online-to-batch conversion bounds just with respect to the negatives. We clarify this further in the discussion.

\subsection{Regret Bounds in the Modified Framework}
We prove the following lemma which will help us in instantiating our online-to-batch conversion proofs.
\begin{lem}
\label{lem:regret-modified}
For the surrogate losses of \preck and pAUC, we have $R(T, s) \leq \sqrt s\cdot R(T)$
\end{lem}
\begin{proof}
The only thing we need to do is analyze one time step for changes in the Lipschitz constant. Fix a time step $t$ and let $\Zb_t = \bc{\x_{t,1},\x_{t,2},\ldots,\x_{t,s}}$. Also, let $g_t(\w,i) := \ell_\P(\Zb_1,\ldots,\Zb_{t-1},\x_{t,{1:i}},\w)$ for any $i = 1 \ldots s$ (note that this gives us $g_t(\w,s) = \ell_\P(\Zb_1,\ldots,\Zb_t,\w)$). Also let us abuse notation to denote $g_t(\w,0) := \ell_\P(\Zb_1,\ldots,\Zb_{t-1},\w) = g_{t-1}(\w,s)$. Let the Lipschitz constant in the model with batch size $s$ be denoted as $G^s_t$. Thus, we have $G^1_t = G_t$, the Lipschitz constant for the problem in the original model (i.e. for $s=1$). Then we have, for any $\w,\w'\in\W$,
\begin{eqnarray*}
\abs{\L_t(\w) - \L_t(w')} &=& \abs{\ell_\P(\Zb_{1:t},\w) - \ell_\P(\Zb_{1:t-1},\w) - \ell_\P(\Zb_{1:t},\w') + \ell_\P(\Zb_{1:t-1},\w')}\\
						  &=& \abs{g_t(\w,s) - g_t(\w,0) - g_t(\w',s) + g_t(\w',0)}\\
						  &=& \abs{\sum_{i=1}^s g_t(\w,i) - g_t(\w,i-1) - g_t(\w',i) + g_t(\w',i-1)}\\
						  &\leq& \sum_{i=1}^s \abs{g_t(\w,i) - g_t(\w,i-1) - g_t(\w',i) + g_t(\w',i-1)}\\
						  &\leq& \sum_{i=1}^s G_t\norm{\w-\w'} = G_t\cdot s\norm{\w-\w'},
\end{eqnarray*}
where the first inequality follows by triangle inequality and the second inequality follows by a repeated application of the Lipschitz property of these loss functions in the original online model (i.e. with batch size $s = 1$). This establishes the Lipschitz constant in this model as $G_t^s \leq s \cdot G_t$. Now, the usual FTRL analysis gives us the following (note that there are only $T/s$ time steps now)
\[
\sum_{t=1}^{T/s} \L_t(\vecw_t) \leq \ell_\P(\vecx_{1:T}, y_{1:T}, \vecw_*) + \frac{\eta}{2}\|\vecw_*\|_2^2 + \frac{\sum_{t=1}^{T/s}(G^s_t)^2}{\eta} \leq \ell_\P(\vecx_{1:T}, y_{1:T}, \vecw_*) + 2s\|\vecw_*\|_2\sqrt{\sum_{t=1}^{T/s}G_t^2},
\]
by setting $\eta$ appropriately. Now, for \preck, $G_t \leq 8$. Thus, we have
\[
\frac{1}{T}\sum_{t=1}^{T/s} \L_t(\vecw_t) \leq \frac{1}{T}\ell_\P(\vecx_{1:T}, y_{1:T}, \vecw_*) + 6\|\vecw_*\|_2\sqrt\frac{s}{T},
\]
which establishes the result for \preck. Similarly, for pAUC, we can show that the regret in the batch model does not worsen by more than a factor of $\sqrt s$.
\end{proof}

\subsection{Online-to-batch Conversion for pAUC}
We will consider the \tpmb setting where negative points come as a stream and positive points reside in an in-memory buffer. At each trial $t$, the learner receives a batch of $s$ negative points $\Xbb^-_t = \{\x^-_{t,1}, \ldots, \x^-_{t,s}\}$ (we shall assume throughout, for simplicity, that $s\beta$ is an integer). Let us denote the loss w.r.t all the positive points in the buffer by $\phi_+: \W \times \R \rightarrow [0, B]$. $\phi_+$ is defined using a loss function $g(\cdot)$ such as hinge loss or logistic loss as
\[
\phi_+(\w, c) = \frac{1}{B}\sum_{i=1}^B g(\w^\top\x^+_i - c)
\]
For sake of brevity, we will abbreviate $\phi_+(\w,c)$ as $\phi_+(c)$, the reference to $\w$ being clear from context. We assume that $\phi_+$ is monotonically increasing (as is the case for hinge loss and logistic regression) and bounded i.e. for some fixed $B > 0$, we have, for all $\w \in \W, c \in \R$, $0 \leq \phi_+(\w,c) \leq B$. The empirical (unnormalized) partial AUC loss for a model $\w \in \W \subseteq \R^d$ over the negative points received in $t$ trials is then given by
\[
\tilde\ell_\paucc(\Xbb^-_{1:t}, \w) \,=\, \sum_{\tau=1}^{t} \sum_{q = 1}^s \T^-_{\beta, t}(\x_{\tau, q}^-, \w) \, \phi_+(\w^\top \x_{\tau,q}^-),
\]
where $\T^-_{\beta, t}(\x^-, \w)$ is the (empirical) indicator function that is turned on whenever $\x^-$ appears in the top-$\beta$ fraction of all the negatives seen till now, ordered by $\w$, i.e. $\T_{\beta,t}^-(\x^-,\w) = 1$ whenever $\abs{\bc{\tau \in [t], q \in [s] : \w^\top\x^->\w^\top\x^-_{\tau, q}}} \leq t s\beta$. We similarly define a population version of this empirical loss function as
\[
\Pff_\paucc(\w) \,=\, \Expp{\x^-}{ \T^-_{\beta}(\x^-, \w) \, \phi_+(\w^\top \x^-)},
\]
where $\T^-_{\beta}(\x^-, \w)$ is the population indicator function with $\T_{\beta}^-(\x^-,\w) = 1$ whenever $\Probb_{\widetilde{\x}^-}\big(\w^\top\widetilde{\x}^- > \w^\top\x^-\big) \leq \beta$. Also, we define $\L_t(\w) = \ell_\paucc(\Xbb^-_{1:t},\w) \,-\, \ell_\paucc(\Xbb^-_{1:t-1},\w)$, with the regret of a learning algorithm that generates an ensemble of models $\w_1,\w_2,\ldots,\w_{T/s} \in \W \subseteq \R^d$ upon receiving $T/s$ batches of negative points $\Xbb^-_{1:T/s}$ defined as:
\[
R(T,s) = \frac{1}{T}\sum_{t=1}^{T/s}\L_t(\w_t) - \mathop{\arg\min}_{\w\in \W}\frac{1}{T}\tilde\ell_\paucc(\Xbb^-_{1:T/s}, \w).
\]
Define $\beta_t = \Expp{\x^-}{\T^-_{\beta, t-1}(\x^-, \w_t)}$ as the fraction of the population that can appear in the top $\beta$ fraction of the set of points seen till now, i.e. the fraction of the population for which the empirical indicator function is turned on, and
\[
\Q_t(\w) = \Expp{\x^-}{\T^-_{\beta, t-1}(\x^-, \w) \, \phi_+(\w^\top \x^-)}
\]
as the population partial AUC computed with respect to the empirical indicator function $\T^-_{\beta, t-1}$ (note that the population risk functional $\Pff_\paucc(\w)$ is computed with respect to $\T^-_{\beta}(\x^-, \w)$, the population indicator function instead). We will also find it useful to define the following conditional expectation.
\[
\exxL_t(\w) = \Expp{\Xbb^-_t}{\L_t(\w)\,|\,\Xbb^-_{1:t-1}}.
\]
We now present a proof sketch of the online-to-batch conversion result in Theorem~\ref{thm:otb-main} for pAUC.
\begin{thm}[Online-to-batch Conversion for pAUC]
\label{thm:otb-pauc}
Suppose the sequence of negative points $\x^-_1, \ldots, \x^-_T$ is generated i.i.d.. Let us partition this sequence into $T/s$ batches of size $s$ and let $\w_1,\w_2,\ldots,\w_{T/s}$ be an ensemble of models generated by an online learning algorithm upon receiving these $T/s$ batches. Suppose the online learning algorithm has a guaranteed regret bound $R(T,s)$. Then for $\overline{\w} = \frac{1}{T/s} \sum_{t=1}^{T/s} \w_t$, any $\w^\ast \in \W \subseteq \R^d$, $\epsilon \in (0, 1]$ and $\delta > 0$, with probability at least $1 - \delta$,
\[
\Pff_\paucc(\overline\w) \,\leq\, (1+\epsilon)\Pff_\paucc(\w^\ast) \,+\, \frac{1}{\beta}R(T,s)  \,+\, e^{-\Om{s\epsilon^2}} \,+\,  \softO{\sqrt{\frac{s\ln(1/\delta)}{T}}}.
\]
In particular, setting $s =\tilde{\cal O}(\sqrt T)$ and $\epsilon = \sqrt[4]{\rfrac{1}{T}}$ gives us, with probability at least $1 - \delta$,
\[
\Pff_\paucc(\overline\w) \,\leq\, \Pff_\paucc(\w^\ast) \,+\, \frac{1}{\beta}R(T,\sqrt{T}) \,+\, \softO{\sqrt[4]{\frac{\ln(1/\delta)}{T}}}.
\]
\end{thm}
\begin{proof}[Proof (Sketch)]
Fix $\epsilon \in (0, 0.5]$. We wish to bound the difference
\begin{align}
\label{eq:otb-diff}
(1-\epsilon)s\beta  \sum_{t=1}^{T/s} \Pff_\paucc(\w_t)\,-\,  T\beta \Pff_\paucc(\w_\ast)
\end{align}
and do so by decomposing \eqref{eq:otb-diff} into four terms as shown below.
\[
\text{\eqref{eq:otb-diff}} \leq \sum_{t=1}^{T/s}RE_t(\w_t) \,+\, MC(\w_{1:T/s}) \,+\, R(\w_{1:T/s}) \,+\, UC(\w_\ast),
\]
where we have
\begin{align}
UC(\w_\ast) &=& \tilde\ell_\paucc(\Xbb^-_{1:T/s}, \w_\ast) \,-\,  T\beta \Pff_\paucc(\w_\ast) \tag{Uniform Convergence Term}\\
R(\w_{1:T/s}) &=& \sum_{t=1}^{T/s} \L_t(\w_t) \,-\, \sum_{t=1}^{T/s} \L_t(\w_\ast) \tag{Regret Term}\\
MC(\w_{1:T/s}) &=& \sum_{t=1}^{T/s} \exxL_t(\w_t) \,-\, \sum_{t=1}^{T/s} \L_t(\w_t) \tag{Martingale Convergence Terms}\\
RE_t(\w_t) &=& (1-\epsilon)s\beta \Pff_\paucc(\w_t) \,-\, \exxL_t(\w_t) \tag{Residual Error Terms}
\end{align}
Note that the above has used the fact that $\tilde\ell_\paucc(\Xbb^-_{1:T/s}, \w_\ast) = \sum_{t=1}^{T/s} \L_t(\w_\ast)$.

We will bound these terms in order below. First we look at the term $UC(\w_\ast)$. Bounding this simply requires a batch generalization bound of the form we prove in Theorem~\ref{thm:gen-bound-rank}. Thus, we can show, that with probability $1 - \delta/3$, we have
\[
UC(\w_\ast) \leq \O{\sqrt{T\log(1/\delta)}}.
\]
We now move on the term $R(\w_{1:T/s})$. This is simply bounded by the regret of the ensemble $\w_{1:T/s}$. This gives us
\[
R(\w_{1:T/s}) \leq T\cdot R(T,s).
\]
The next term we bound is $MC(\w_{1:T/s})$. Note that by definition of $\exxL_t(\w)$, if we define
\[
v_t = \exxL_t(\w_t) - \L_t(\w_t),
\]
then the terms $\bc{v_t}$ form a martingale difference sequence. Since $\big|\exxL_t(\w_t) - \L_t(\w_t)\big| \leq \O s$, we get, by an application of the Azuma-Hoefding inequality, with probability at least $1 - \delta/3$,
\[
MC(\w_{1:T/s}) \leq \O{s\sqrt{\frac{T}{s}\ln\frac{1}{\delta}}} = \O{\sqrt{sT\ln(1/\delta)}}.
\]
The last step requires us to bound the residual term $RE_t(\w_t)$ which will again require uniform convergence techniques. We shall show, that with probability, at least $1 - (\delta\cdot s/3T)$, we have 
\[
\beta_t \geq \beta - \softO{\sqrt\frac{\log\frac{1}{\delta}}{s(t-1)}}.
\]
This shall allow us to show that with the same probability, we have
\[
\Q_{t}(\w_t) - \Pff_\paucc(\w_t) \leq \softO{\sqrt\frac{\log\frac{1}{\delta}}{s(t-1)}}.
\]
The last ingredient in the proof shall involve showing that the following holds for any $\epsilon > 0$
\[
\exxL_t(\w_t) \geq (1-\epsilon)s\beta_t\Q_t(\w_t) - \Om{s\exp(-s\beta_t^2\epsilon^2)}
\]
Combining the above with a union bound will show us that, with probability at least $1 - \delta/3$,
\[
\sum_{i=1}^{T/s}RE_t(\w_t) \leq \O{T\exp(-s\epsilon^2)} + \softO{\sqrt{sT\log(1/\delta)}}
\]
A final union bound and some manipulations would then establish the claimed result.
\end{proof}
\section{Proof of Theorem~\ref{thm:erm}}
\label{app:proof-genbound-pauc}
The proof proceeds in two parts: the first part uses the fact that the \spmb method essentially simulates the GIGA method of \cite{zinkevich} with the non-decomposable loss function and the second part uses the uniform convergence properties of the loss function to establish the error bound. To proceed, let us set up some notation. Consider the $e\th$ epoch of the \spmb algorithm. Let us denote the set of points considered in this epoch by $X_e = \bc{x^e_1,\ldots,x^e_s}$. With this notation it is clear that the \spmb algorithm can be said to be performing online gradient descent with respect to the instantaneous loss functions $\L_e(\w) = \L(X_e,\w) := \ell_\P(\x^e_{1:s},y^e_{1:s},\vecw)$.

Since the loss function $\L_e(\w)$ is convex, the standard analysis for online convex optimization would apply under mild boundedness assumptions on the domain and the (sub)gradients of the loss function. Since there are $n/s$ epochs (assuming for simplicity that $n$ is a multiple of $s$), this allows us to use the standard regret bounds \cite{zinkevich} to state the following:
\[
\frac{s}{n}\sum_{e=1}^{n/s}\L_e(\w_e) \leq \frac{s}{n}\sum_{e=1}^{n/s}\L_e(\w_\ast) + \O{\sqrt{\frac{s}{n}}}.
\]
Now we will invoke uniform convergence properties of the loss function. However, doing so requires clarifying certain aspects of the problem setting. The statement of Theorem~\ref{thm:erm} assumes only a random ordering of training data points whereas uniform convergence properties typically require i.i.d. samples. We reconcile this by noticing that all our uniform convergence proofs use the Hoeffding's lemma to establish statistical convergence and that the Hoeffding's lemma holds when random variables are sampled without replacement as well (e.g. see \cite{Serfling74}). Since a random ordering of the data provides, for each epoch, a uniformly random sample without replacement, we are able to invoke the uniform convergence proofs.

Thus, if we denote $\L(\vecw) := \ell_\P(\x_{1:n},y_{1:n},\w)$, then by using the uniform convergence properties of the loss function, for every $e$, with probability at least $1 -\frac{s\delta}{n}$, we have $\L_e(\w_e) \geq \L(\w_e) - \alpha\br{s,\frac{s\delta}{n}}$ as well as $\L_e(\w_\ast) \leq \L(\w_\ast) + \alpha\br{s,\frac{s\delta}{n}}$. Applying the union bound and Jensen's inequality gives us, with probability at least $1 - \delta$, the desired result:
\[
\L(\overline\w) \leq \frac{s}{n}\sum_{e=1}^{n/s}\L(\w_e) \leq \L(\w_\ast) + 2\alpha\br{s,\frac{s\delta}{n}} + \O{\sqrt{\frac{s}{n}}}.
\]

We note that we can use similar arguments as above to give error bounds for the \tpmb procedure as well.
Suppose $\bx^+_{1:s_+}$ and $\bx^-_{1:s_-}$ are the positive and negative points sampled in the process (note that here the number of positive and negatives points (i.e. $s_+$ and $s_-$ respectively) are random quantities as well). Also suppose $\x^+_{1:n_+}$ and $\x^-_{1:n_-}$ are the positive and negative points in the population. Then recall that Definition~\ref{def:uc} requires, for a uniform (but possibly without replacement) sample,
\[
\underset{\w\in\W}{\sup}\abs{\ell_\P(\x^+_{1:n_+},\x^-_{1:n_-},\w)-\ell_\P(\bx^+_{1:s_+}, \bx^-_{1:s_-},\w)} \leq \softO{\alpha(s,\delta)}.
\]
To prove bounds for \tpmb, we require that for arbitrary choice of $s_+, s_- \geq \Om{s}$, when $\bx^+_{1:s_+}$ and $\bx^-_{1:s_-}$ are chosen separately and uniformly (but yet again possibly without replacement) from $\x^+_{1:n_+}$ and $\x^-_{1:n_-}$ respectively, we still obtain a similar result as above.
Since the first pass and each epoch of the second pass provide such a sample, we can use this result to prove error bounds for the \tpmb procedure. We defer the detailed arguments for such results to the full version of the paper.

We however note that the proof of Theorem~\ref{thm:gen-bound-rank} below does indeed prove such a result for the pAUC loss function by effectively proving (see Section~\ref{sec:point-conv-pauc}) the following two results
\begin{eqnarray*}
\underset{\w\in\W}{\sup}\abs{\ell_\P(\x^+_{1:n_+},\bx^-_{1:s_-},\w)-\ell_\P(\bx^+_{1:s_+}, \bx^-_{1:s_-},\w)} &\leq& \softO{\alpha(s,\delta)}\\
\underset{\w\in\W}{\sup}\abs{\ell_\P(\x^+_{1:n_+},\x^-_{1:n_-},\w)-\ell_\P(\x^+_{1:n_+},\bx^-_{1:s_-},\w)} &\leq& \softO{\alpha(s,\delta)}.
\end{eqnarray*}

\section{Uniform Convergence Bounds for Partial Area under the ROC Curve}
\label{app:uc-proof-pauc}
In this section we present a proof sketch of Theorem~\ref{thm:gen-bound-rank} which we restate below for convenience.
\begin{thm}
Consider any convex, monotonic and Lipschitz classification surrogate $\phi : \R \rightarrow \R_+$. Then the loss function for the $(0,\beta)$-partial AUC performance measure defined as follows exhibits uniform convergence at the rate $\alpha(s) = \softO{1/{\sqrt s}}$:
\[
\ell_\P(\x_{1:n},y_{1:n},\w) = \frac{1}{\beta n_+n_-}\sum_{i=1}^n\Ind{y_i > 0}\sum_{j=1}^n\Ind{y_j < 0}\T^-_{\beta, n}(\x_j,\w)\phi\br{\w^\top(\x_i-\x_j)},
\]
where $n_+ = \abs{\bc{i : y_i > 0}}$ and $n_- = \abs{\bc{i : y_i < 0}}$.
\end{thm}

\begin{proof}[Proof (Sketch)]
We shall use the notation $\hT^-_{\beta,s}$ to denote the indicator function for the top $\beta$ fraction of the negative elements in the smaller sample of size $s$. Thus, over the smaller sample $(\bx_1,\by_1)\ldots(\bx_s,\by_s)$, the pAUC is calculated as
\[
\ell_\P(\bx_{1:s},\by_{1:s},\w) = \frac{1}{\beta s_+s_-}\sum_{i=1}^s\Ind{\by_i > 0}\sum_{j=1}^s\Ind{\by_j < 0}\hT^-_{\beta, s}(\bx_j,\w)\phi\br{\w^\top(\bx_i-\bx_j)}.
\]
Our goal would be to show that with probability at least $1 - \delta$, for all $\w \in \W$
\[
\abs{\ell_\P(\x_{1:n},y_{1:n},\w) - \ell_\P(\bx_{1:s},\by_{1:s},\w)} \leq \softO{\frac{1}{\sqrt s}}
\]
We shall demonstrate this by establishing the following three statements:
\begin{enumerate}
	\item For any fixed $\w\in\W$, w.h.p., we have $\abs{\ell_\P(\x_{1:n},y_{1:n},\w) - \ell_\P(\bx_{1:s},\by_{1:s},\w)} \leq \softO{\frac{1}{\sqrt s}}$
	\item For any two $\w,\w'\in\W$, we have $\abs{\ell_\P(\x_{1:n},y_{1:n},\w)-\ell_\P(\x_{1:n},y_{1:n},\w')} \leq \O{\norm{\w-\w'}_2}$
	\item For any two $\w,\w'\in\W$, we have $\abs{\ell_\P(\bx_{1:s},\by_{1:s},\w)-\ell_\P(\bx_{1:s},\by_{1:s},\w')} \leq \O{\norm{\w-\w'}_2}$
\end{enumerate}
With these three results established, we would be able to conclude the proof by an application of a standard covering number argument. We now prove these three statements in parts.

\subsection{Part 1: Pointwise Convergence for pAUC}
\label{sec:point-conv-pauc}
Fix a predictor $\w\in\W$ and $S_+$ and $S_-$ denote the set of positive and negative samples. We shall assume that $s_+,s_- \geq \Om{s}$ which holds with high probability. Denote, for any $\x_i$ such that $y_i > 0$,
\[
\ell^+(\x_i,\w) = \frac{1}{\beta n_-}\sum_{j=1}^n\Ind{y_j < 0}\T^-_{\beta, n}(\x_j,\w)\phi\br{\w^\top(\x_i-\x_j)},
\]
and for any $\bx_i \in S_+$,
\[
\ell^+_{S_-}(\bx_i,\w) = \frac{1}{\beta s_-}\sum_{j=1}^s\Ind{\by_j < 0}\hT^-_{\beta, s}(\bx_j,\w)\phi\br{\w^\top(\bx_i-\bx_j)}.
\]
Notice that $\ell_\P(\bx_{1:s},\by_{1:s},\w) = \frac{1}{n_+}\sum_{i=1}^n\Ind{y_i > 0}\ell^+(\x_i,\w)$ and $\ell_\P(\bx_{1:s},\by_{1:s},\w) = \frac{1}{s_+}\sum_{i=1}^s\Ind{\by_i > 0}\ell^+_{S_-}(\bx_i,\w)$.  We shall now show the following holds w.h.p. over $S_-$:
\begin{enumerate}
	\item For any $\x_i$ such that $y_i > 0$, $\abs{\ell^+(\x_i,\w) - \ell^+_{S_-}(\x_i,\w)} \leq \softO{\frac{1}{\sqrt s}}$.
	\item $\frac{1}{n_+}\abs{\sum_{i=1}^n\Ind{y_i > 0}\ell^+(\x_i,\w) - \Ind{y_i > 0}\ell^+_{S_-}(\x_i,\w)} \leq \softO{\frac{1}{\sqrt s}}$.
	\item $\abs{\frac{1}{n_+}\sum_{i=1}^n\Ind{y_i > 0}\ell^+_{S_-}(\x_i,\w) - \frac{1}{s_+}\sum_{i=1}^s\Ind{\by_i > 0}\ell^+_{S_-}(\bx_i,\w)} \leq \softO{\frac{1}{\sqrt s}}$.
\end{enumerate}
The second part follows from the first part by an application of the triangle inequality. The third part also can be shown to hold by an application of Hoeffding's inequality and other arguments. This leaves the first part for which we provide a proof in the full version of the paper.

\subsection{Parts 2 and 3: Establishing an $\epsilon$-net for pAUC}
For simplicity, we assume that the domain is finite. This does not affect the proof in any way since it still allows the domain to be approximated arbitrary closely by an $\epsilon$-net of (arbitrarily) large size. However, we note that we can establish the same result for infinite domains as well, but choose not to for sake of simplicity. We prove the second part, the proof of the first part being similar. We have
\begin{eqnarray*}
\abs{\ell_\P(\x_{1:n},y_{1:n},\w) - \ell_\P(\x_{1:n},y_{1:n},\w')} &=& \frac{1}{s_+}\abs{\sum_{i=1}^s\Ind{\by_i > 0}\ell^+_{S_-}(\x_i,\w) - \Ind{\by_i > 0}\ell^+_{S_-}(\x_i,\w')}\\
&\leq& \frac{1}{s_+}\sum_{i=1}^s\abs{\Ind{\by_i > 0}\br{\ell^+_{S_-}(\x_i,\w) - \ell^+_{S_-}(\x_i,\w')}}\\
&\leq&	\O{\norm{\w-\w'}_2},
\end{eqnarray*}
using Lemma~\ref{lem:list} with $g(a) = \phi(\w^\top\x_i - a)$ and $c_i = 0$. This concludes the proof.
\end{proof}

\section{Methodology for implementing \spmb and \tpmb for pAUC tasks}
\label{app:exps-pauc}
In this section we clarify the mechanisms used to implement the \spmb and \tpmb routines. Going as per the dataset statistics (see Table~\ref{tab:dataset-stats}), we will consider the variant of the \tpmb routine with the positive class as the rare class. Recall the definition of the surrogate loss function for pAUC \eqref{eq:pauc_cvx}
\[
\ell_{\text{pAUC}}(\w) = \sum_{i:y_i>0}\sum_{j:y_j<0}\mathbb{T}^-_{\beta, t}(\vecx_j, \vecw)\cdot h(\vecx_i^\top\vecw- \vecx_j^\top\vecw).
\]
We now rewrite this in a slightly different manner. Define, for any $i : y_i > 0$
\[
\ell^+_{S_-}(x_i,\w) = \sum_{j: y_j < 0}\T^-_{\beta, t}(\x_j,\w)\cdot h(\x_i^\top\w - \x_j^\top\w),
\]
so that we can write $\ell_{\text{pAUC}}(\w) = \sum_{i:y_i>0}\ell^+_{S_-}(x_i,\w)$. This shows that a subgradient to $\ell_{\text{pAUC}}(\w)$ can be found by simply finding and summing up, subgradients for $\ell^+_{S_-}(x_i,\w)$. For now, fix an $i$ such that $y_i > 0$ and define $g(a) = h(\vecx_i^\top\vecw - a)$. Using the properties of the hinge loss function, it is clear that $g(a)$ is an increasing function of $a$. Since $\ell^+_{S_-}(x_i,\w)$ is defined on the top ranked $\lceil \beta t_- \rceil$ negatives, we can, using the monotonicity of $g(\cdot)$, equivalently write it as follows. Let $\ZZ_\beta = \binom{S_-}{\lceil \beta t_- \rceil}$ be the set of all sets of negative points of negative training points of size $\lceil \beta t_- \rceil$. Then we can write
\[
\ell^+_{S_-}(x_i,\w) = \underset{S \in \ZZ_\beta}{\max}\sum_{\x^- \in S} g(\x^{-\top}\w)
\]
Since the maximum in the above formulation is achieved at $S = \bc{j: y_j < 0, \T^-_{\beta, t}(\x_j,\w) = 1}$, by Danskin's theorem (see, for example \cite{Bertsekas04}), we get the following result: let $\v_{ij} \in \delta h(\x_i^\top\w - \x_j^\top\w)$ be a subgradient to the hinge loss function, then for the following vector
\[
\v_i := \sum_{j: y_j < 0}\T^-_{\beta, t}(\x_j,\w)\cdot\v_{ij},
\]
we have $\v_i \in \delta\ell^+_{S_-}(x_i,\w)$ and consequently, for $\v := \sum_{i:y_i>0}\v_i$, we have $\v \in \delta\ell_{\text{pAUC}}(\w)$. This gives us a straightforward way to implement \spmb: for each epoch, we take all the negatives in that epoch, filter out the top $\beta$ fraction of them according to the scores assigned to them by the current iterate $\w_e$ and then calculate the (sub)gradients between all the positives in that epoch and these filtered negatives. This takes at most $\O{s \log s}$ time per epoch.

\end{document}